%% file: fullpaper_RO_arXiv.tex
\documentclass[3p]{elsarticle_my}

\usepackage{hyperref}

\journal{ArXiv}

\usepackage[latin1]{inputenc}
\usepackage{graphicx}
\usepackage[usenames]{color}
\usepackage{amsmath,amssymb,amsthm}
\usepackage{colortbl}
\usepackage{caption}
\usepackage{subcaption}
\usepackage{multirow}

\usepackage{algorithm}
\usepackage{algorithmic}

\DeclareMathOperator*{\argmax}{argmax}

\usepackage{tikz}
\definecolor{myblue}{RGB}{0,106,214}
\usetikzlibrary{calc}
\usetikzlibrary{arrows}

\usepackage{booktabs}
\usepackage[para]{threeparttable}
\usepackage{multirow}

\newtheorem{theorem}{Theorem}

\newtheorem{example}{Example}

\newtheorem{definition}{Definition}

\usepackage{color}

\usepackage{url}

\begin{document}

\begin{frontmatter}
\title{Theoretical Foundations of Forward Feature Selection Methods based on Mutual Information}

\author[ist,epfl]{Francisco Macedo\corref{cor1}}
\ead{francisco.macedo@epfl.ch}

\author[ist]{M. Ros\'{a}rio Oliveira\corref{mycorrespondingauthor}}
\cortext[mycorrespondingauthor]{Corresponding author}
\ead{rosario.oliveira@tecnico.ulisboa.pt}

\author[ist]{Ant\'{o}nio Pacheco\corref{cor3}}
\ead{apacheco@math.tecnico.ulisboa.pt}

\author[ist2]{Rui Valadas\corref{cor4}}
\ead{rui.valadas@tecnico.ulisboa.pt}

\address[ist]{CEMAT and Department of Mathematics, Instituto Superior T\'{e}cnico, Universidade de Lisboa, Av. Rovisco Pais, 1049-001 Lisboa, Portugal}
\address[epfl]{EPF Lausanne, SB-MATHICSE-ANCHP, Station 8, CH-1015 Lausanne, Switzerland}
\address[ist2]{IT and Departament of Electrical and Computer Engineering, Instituto Superior T\'{e}cnico, Universidade de Lisboa, Av. Rovisco Pais, 1049-001 Lisboa, Portugal}

\begin{abstract}
	
	Feature selection problems arise in a variety of applications, such as microarray analysis, clinical prediction, text categorization, image classification and face recognition, multi-label learning, and classification of internet traffic. Among the various classes of methods, forward feature selection methods based on mutual information have become very popular and are widely used in practice. However, comparative evaluations of these methods have been limited by being based on specific datasets and classifiers. In this paper, we develop a theoretical framework that allows evaluating the methods based on their theoretical properties. Our framework is grounded on the properties of the target objective function that the methods try to approximate, and on a novel categorization of features, according to their contribution to the explanation of the class; we derive upper and lower bounds for the target objective function and relate these bounds with the feature types. Then, we characterize the types of approximations taken by the methods, and analyze how these approximations cope with the good properties of the target objective function. Additionally, we develop a distributional setting designed to illustrate the various deficiencies of the methods, and provide several examples of wrong feature selections. Based on our work, we identify clearly the methods that should be avoided, and the methods that currently have the best performance.
	
\end{abstract}



\end{frontmatter}


\input{intro.tex}

\input{section2.tex}

\input{section3.tex}

\input{section4.tex}

\input{section5_RO.tex}

\input{section6_ArXiv.tex}

\input{concs.tex}
\input{appendix.tex}


\section*{Acknowledgements}

Research was partially sponsored by national funds through the Fundação Nacional
para a Ciência e Tecnologia (FCT), Portugal, under the projects PEstOE/MAT/UI0822/2014
and  PTDC/EEI-TEL/5708/2014. F. Macedo was funded by the FCT PhD grant SFRH/BD/51930/2012.

\bibliography{biblifinal3}

\end{document}

%% file: intro.tex
\section{Introduction}
\label{sec:intro}

In an era of data abundance, of a complex nature, it is of utmost importance to extract from the data useful and valuable knowledge for real problem solving. Companies seek in the pool of available information commercial value that can leverage them among competitors or give support for making strategic decisions. One important step in this process is the selection of relevant and non-redundant information in order to clearly define the problem at hand and aim for its solution \citep[see][]{bolon2015recent}. 

Feature selection problems arise in a variety of applications, reflecting their importance. Instances can be found in: microarray analysis \citep[see][]{xing2001feature,saeys2007review,bolon2013review,li2004comparative,liu2002comparative}, clinical prediction \citep[see][]{bagherzadeh2015tutorial,li2004comparative,liu2002comparative}, text categorization \citep[see][]{yang1997comparative,rogati2002high,varela2013empirical,Khan:2016:SWI:2912588.2912779}, image classification and face recognition \citep[see][]{bolon2015recent}, multi-label learning \citep[see][]{schapire2000boostexter,crammer2002new}, and classification of internet traffic \citep[see][]{pascoal2012robust}.

Feature selection techniques can be categorized as classifier-dependent (\textit{wrapper} and \textit{embedded} methods) and classifier-independent (\textit{filter} methods). Wrapper methods \citep{kohavi1997wrappers} search the space of feature subsets, using the classifier accuracy as the measure of utility for a candidate subset. There are clear disadvantages in using such approach. The computational cost is huge, while the selected features are specific for the considered classifier. Embedded methods \citep[Ch. 5]{guyon2008feature} exploit the structure of specific classes of classifiers to guide the feature selection process. In contrast, filter methods \citep[Ch. 3]{guyon2008feature} separate the classification and feature selection procedures, and define a heuristic ranking criterion that acts as a measure of the classification accuracy. 

Filter methods differ among them in the way they quantify the benefits of including a particular feature in the set used in the classification process. Numerous heuristics have been suggested. Among these, methods for feature selection that rely on the concept of \textit{mutual information} are the most popular. Mutual information (MI) captures linear and non-linear association between features, and is strongly related with the concept of \textit{entropy}. Since considering the complete set of candidate features is too complex, filter methods usually operate sequentially and in the forward direction, adding one candidate feature at a time to the set of selected features. Here, the selected feature is the one that, among the set of candidate features, maximizes an objective function expressing the contribution of the candidate to the explanation of the class. A unifying approach for characterizing the different forward feature selection methods based on MI has been proposed by \cite{Brown:2012:CLM:2188385.2188387}. \cite{vergara2014review} also provide an overview of the different feature selection methods, adding a list of open problems in the field.

Among the forward feature selection methods based on MI, the first proposed group \citep{Battiti94usingmutual,Peng05featureselection,MR2422423,claudia} is constituted by methods based on assumptions that were originally introduced by \cite{Battiti94usingmutual}. These methods attempt to select the candidate feature that leads to: maximum \textit{relevance} between the candidate feature and the class; and minimum \textit{redundancy} of the candidate feature with respect to the already selected features. Such redundancy, which we call \textit{inter-feature redundancy}, is measured by the level of association between the candidate feature and the previously selected features. Considering inter-feature redundancy in the objective function is important, for instance, to avoid later problems of collinearity. In fact, selecting features that do not add value to the set of already selected ones in terms of class explanation, should be avoided.

A more recently proposed group of methods based on MI considers an additional term, resulting from the accommodation of possible dependencies between the features given the class \citep{Brown:2012:CLM:2188385.2188387}. This additional term is disregarded by the previous group of filter methods. Examples of methods from this second group are the ones proposed by: \cite{LinT06,YangM99,MR2248026}. The additional term expresses the contribution of a candidate feature to the explanation of the class, when taken together with already selected features, which corresponds to a \textit{class-relevant redundancy}. The effects captured by this type of redundancy are also called \textit{complementarity} effects.

In this work we provide a comparison of forward feature selection methods based on mutual information using a theoretical framework. The framework is independent of specific datasets and classifiers and, therefore, provides a precise evaluation of the relative merits of the feature selection methods; it also allows unveiling several of their deficiencies. Our framework is grounded on the definition of a target (ideal) objective function and of a categorization of features according to their contribution to explanation of the class. We derive lower and upper bounds for the target objective function and establish a relation between these bounds and the feature types. The categorization of features has two novelties regarding previous works: we introduce the category of fully relevant features, features that fully explain the class together with already selected features, and we separate non-relevant features into irrelevant and redundant since, as we show, these categories have different properties regarding the feature selection process.

This framework provides a reference for evaluating and comparing actual feature selection methods. Actual methods are based on approximations of the target objective function, since the latter is difficult to estimate. We select a set of methods representative of the various types of approximations, and discuss the various drawbacks they introduced. Moreover, we analyze how each method copes with the good properties of the target objective function. Additionally, we define a distributional setting, based on a specific definition of class, features, and a novel performance metric; it provides a feature ranking for each method that is compared with the ideal feature ranking coming out of the theoretical\textcolor{red}{ly} framework. The setting was designed to challenge the actual feature selection methods, and illustrate the consequences of their drawbacks. Based on our work, we identify clearly the methods that should be avoided, and the methods that currently have the best performance.

Recently, there has been several attempts to undergo a theoretical evaluation of forward feature selection methods based on MI. \cite{Brown:2012:CLM:2188385.2188387} and \cite{vergara2014review} provide an interpretation of the objective function of actual methods as approximations of a target objective function, which is similar to ours. However, they do not study the consequences of these approximations from a theoretical point-of-view, i.e. how the various types of approximations affect the good properties of the target objective function, which is the main contribution of our work. Moreover, they do not cover all types of feature selection methods currently proposed. \cite{claudiapaper} evaluated methods based on a distributional setting similar to ours, but the analysis is restricted to the group of methods that ignore complementarity, and again, does not address the theoretical properties of the methods.

The rest of the paper is organized as follows. We introduce some background on entropy and MI in Section \ref{sec:entropymi}. This is followed, in Section \ref{sec:tmicmi}, by the presentation of the main concepts associated with conditional MI and MI between three random vectors. In Section \ref{sec:problem}, we focus on explaining the general context concerning forward feature selection methods based on MI, namely the target objective function, the categorization of features, and the relation between the feature types and the bounds of the target objective function. In Section \ref{sec:methods}, we introduce representative feature selection methods based on MI, along with their properties and drawbacks. In Section \ref{sec:setting}, we present a distribution based setting where some of the main drawbacks of the representative methods are illustrated, using the minimum Bayes risk as performance evaluation measure to assess the quality of the methods. The main conclusions can be found in Section \ref{sec:concs}.

%% file: section2.tex
\section{Entropy and mutual information}
\label{sec:entropymi}

In this section, we present the main ideas behind the concepts of entropy and mutual information, along with their basic properties. In what follows, $\mathcal{X}$ denotes the support of a random vector $\boldsymbol{X}$. Moreover, we assume the convention $0\ln 0=0$, justified by continuity since $x \ln x \rightarrow 0$ as $x\rightarrow 0^+$.

\subsection{Entropy}
\label{subsec:entropy}

The concept of entropy \citep{MR0026286} was initially motivated by problems in the field of telecommunications. 
Introduced for discrete random variables, the entropy is a measure of uncertainty. In the following, $P(A)$ denotes the probability of $A$. 

\begin{definition}
\label{def:entropydiscrete}
The entropy of a discrete random vector $\boldsymbol{X}$ is:
\begin{equation}
\label{eq:entropydiscrete}
H(\boldsymbol{X})=-\sum_{\boldsymbol{x} \in \mathcal{X}} P(\boldsymbol{X}=\boldsymbol{x}) \ln P(\boldsymbol{X}=\boldsymbol{x}).
\end{equation}
Given an additional discrete random vector $\boldsymbol{Y}$, the conditional entropy of $\boldsymbol{X}$ given $\boldsymbol{Y}$ is
\begin{equation*}
H(\boldsymbol{X}|\boldsymbol{Y})=-\sum_{\boldsymbol{y} \in \mathcal{Y}}\sum_{\boldsymbol{x} \in \mathcal{X}}P(\boldsymbol{X}=\boldsymbol{x}|\boldsymbol{Y}=\boldsymbol{y}) P(\boldsymbol{Y}=\boldsymbol{y}) \ln P(\boldsymbol{X}=\boldsymbol{x}|\boldsymbol{Y}=\boldsymbol{y}).
\end{equation*}
\end{definition}

Note that the entropy of $\boldsymbol{X}$ does not depend on the particular values taken by the random vector but only on the corresponding probabilities. It is clear that entropy is non-negative since each term of the summation in \eqref{def:entropydiscrete} is non-positive. Additionally, the value $0$ is only obtained for a degenerate random variable.

An important property that results from Definition \ref{def:entropydiscrete} is the so-called \textit{chain rule} \citep[Ch. 2]{MR2239987}:
\begin{equation}
H(\boldsymbol{X}_1,...,\boldsymbol{X}_n)=\sum_{i=2}^{n}H(\boldsymbol{X}_i|\boldsymbol{X}_{i-1},...,\boldsymbol{X}_1)+H(\boldsymbol{X}_1),
\label{eq:chainruleall}
\end{equation}
where a sequence of random vectors, such as $(\boldsymbol{X}_1,...,\boldsymbol{X}_n)$ and $(\boldsymbol{X}_{i-1},...,\boldsymbol{X}_1)$ above, should be seen as the random vector that results from the concatenation of its elements.

\subsection{Differential entropy}
\label{subsec:diffentropy}

A logical way to adapt the definition of entropy to the case where we deal with an absolutely continuous random vector is to replace the probability (mass) function of a discrete random vector by the probability density function of an absolutely continuous random vector, as next presented. The resulting concept is called \textit{differential entropy}. We let $f_{\boldsymbol{X}}$ denote the probability density function of an absolutely continuous random vector $\boldsymbol{X}$.

\begin{definition}
\label{def:entropycontinuous}
The differential entropy of an absolutely continuous random vector $\boldsymbol{X}$ is:
\begin{equation}
\label{eq:entropycontinuous}
h(\boldsymbol{X})= -\int_{\boldsymbol{x}\in \mathcal{X}} f_{\boldsymbol{X}}(\boldsymbol{x}) \ln{f_{\boldsymbol{X}}(\boldsymbol{x})} d\boldsymbol{x}.
\end{equation}
Given an additional absolutely continuous random vector $\boldsymbol{Y}$, such that $(\boldsymbol{X},\boldsymbol{Y})$ is also absolutely continuous, the conditional differential entropy of $\boldsymbol{X}$ given $\boldsymbol{Y}$ is
\begin{equation*}
h(\boldsymbol{X}|\boldsymbol{Y})=-\int_{\boldsymbol{y} \in \mathcal{Y}} f_{\boldsymbol{Y}}(\boldsymbol{y}) \int_{\boldsymbol{x} \in \mathcal{X}} f_{\boldsymbol{X}|\boldsymbol{Y}=\boldsymbol{y}}(\boldsymbol{x}) \ln f_{\boldsymbol{X}|\boldsymbol{Y}=\boldsymbol{y}}(\boldsymbol{x}) d\boldsymbol{x}\, d\boldsymbol{y}.
\end{equation*}
\end{definition}

It can be proved \citep[Ch. 9]{MR2239987} that the chain rule \eqref{eq:chainruleall} still holds replacing entropy by differential entropy.

The notation that is used for differential entropy, $h$, is different from the notation used for entropy, $H$. This is justified by the fact that entropy and differential entropy do not share the same properties. For instance, non-negativity does not necessarily hold for differential entropy. Also note that $h(\boldsymbol{X},\boldsymbol{X})$ and $h(\boldsymbol{X}|\boldsymbol{X})$ are not defined given that the pair $(\boldsymbol{X},\boldsymbol{X})$ is not absolutely continuous. Therefore, relations involving entropy and differential entropy need to be interpreted in a different way.

\begin{example}
If $\boldsymbol{X}$ is a random vector, of dimension $n$, following a multivariate normal distribution with mean $\boldsymbol{\mu}$ and covariance matrix $\boldsymbol{\Sigma}$, $\boldsymbol{X}\sim \mathcal{N}_n (\boldsymbol{\mu}, \boldsymbol{\Sigma})$, the value of the corresponding differential entropy is $\frac{1}{2}\ln\left((2\pi e)^n |\boldsymbol{\Sigma}|\right)$ \citep[Ch. 9]{MR2239987}, where $|\boldsymbol{\Sigma}|$ denotes the determinant of $\boldsymbol{\Sigma}$. In particular, for the one-dimensional case, $X\sim \mathcal{N} (\mu, \sigma^2)$, the differential entropy is negative if $\sigma^2 <1/2\pi e$, positive if $\sigma^2 >1/2\pi e$, and zero if $\sigma^2 =1/2 \pi e$. Thus, a zero differential entropy does not have the same interpretation as in the discrete case. Moreover, the differential entropy can take arbitrary negative values. 
\label{ex:normal}
\end{example}

In the rest of the paper, when the context is clear, we will refer to differential entropy simply as entropy.

\subsection{Mutual information}
\label{subsec:mi}

We now introduce mutual information (MI), which is a very important measure since it measures both linear and non-linear associations between random vectors.

\subsubsection{Discrete case}

\begin{definition}\label{def:midiscrete}
The MI between two discrete random vectors $\boldsymbol{X}$ and $\boldsymbol{Y}$ is\textcolor{red}{:}
\begin{equation*}
{\rm MI}(\boldsymbol{X},\boldsymbol{Y})=\sum_{\boldsymbol{x} \in \mathcal{X}} \sum_{\boldsymbol{y} \in \mathcal{Y}} P(\boldsymbol{X}=\boldsymbol{x},\boldsymbol{Y}=\boldsymbol{y}) \ln \frac{P(\boldsymbol{X}=\boldsymbol{x}, \boldsymbol{Y}=\boldsymbol{y})}{P(\boldsymbol{X}=\boldsymbol{x})P(\boldsymbol{Y}=\boldsymbol{y})}.
\end{equation*}
\end{definition}

MI satisfies the following  \citep[cf.][Ch. 9]{MR2239987}:
\begin{align}
  \label{eq:micondproperty}
  {\rm MI}(\boldsymbol{X},\boldsymbol{Y})&= H(\boldsymbol{X})-H(\boldsymbol{X}|\boldsymbol{Y});\\
  \label{eq:midiscnonnegativity}
  {\rm MI}(\boldsymbol{X},\boldsymbol{Y})& \geq 0;\\
  \label{eq:midiscxx}
  {\rm MI}(\boldsymbol{X},\boldsymbol{X})&=H(\boldsymbol{X}).
\end{align}
Equality holds in \eqref{eq:midiscnonnegativity} if and only if $\boldsymbol{X}$ and $\boldsymbol{Y}$ are independent random vectors.

According to \eqref{eq:micondproperty}, ${\rm MI}(\boldsymbol{X},\boldsymbol{Y})$ can be interpreted as the reduction in the uncertainty of $\boldsymbol{X}$ due to the knowledge of $\boldsymbol{Y}$. Note that, applying \eqref{eq:chainruleall}, we also have
\begin{equation}
  \label{eq:micondproperty2}
  {\rm MI}(\boldsymbol{X},\boldsymbol{Y})= H(\boldsymbol{X})+H(\boldsymbol{Y})-H(\boldsymbol{X},\boldsymbol{Y}).
\end{equation}

Another important property that immediately follows from \eqref{eq:micondproperty} is
\begin{equation}
{\rm MI}(\boldsymbol{X},\boldsymbol{Y}) \leq \min (H(\boldsymbol{X}), H(\boldsymbol{Y})).
\label{eq:miupperbound}
\end{equation}

In sequence, in view of \eqref{eq:micondproperty} and \eqref{eq:midiscnonnegativity}, we can conclude that, for any random vectors $\boldsymbol{X}$ and $\boldsymbol{Y}$, 
\begin{equation}
\label{eq:infodonthurt}
H(\boldsymbol{X}|\boldsymbol{Y})\leq H(\boldsymbol{X}).
\end{equation}
This result is again coherent with the intuition that entropy measures uncertainty. In fact, if more information is added, about $\boldsymbol{Y}$ in this case, the uncertainty about $\boldsymbol{X}$ will not increase.

\subsubsection{Continuous case}

\begin{definition}\label{def:micontinuous}
The MI between two absolutely continuous random vectors $\boldsymbol{X}$ and $\boldsymbol{Y}$, such that $(\boldsymbol{X},\boldsymbol{Y})$ is also absolutely continuous, is\textcolor{red}{:}
\begin{equation*}
{\rm MI}(\boldsymbol{X},\boldsymbol{Y})=-\int_{\boldsymbol{y} \in \mathcal{Y}}\int_{\boldsymbol{x} \in \mathcal{X}} f_{\boldsymbol{X},\boldsymbol{Y}}(\boldsymbol{x},\boldsymbol{y}) \ln \frac{f_{\boldsymbol{X},\boldsymbol{Y}}(\boldsymbol{x},\boldsymbol{y})}{f_{\boldsymbol{X}}(\boldsymbol{x})f_{\boldsymbol{Y}}(\boldsymbol{y})} d\boldsymbol{x}\, d\boldsymbol{y}.
\end{equation*}
\end{definition}

It is straight-forward to check, given the similarities between this definition and Definition \ref{def:midiscrete}, that most properties from the discrete case still hold replacing entropy by differential entropy. In particular, the only property from \eqref{eq:micondproperty} to \eqref{eq:midiscxx} that cannot be restated for differential entropy is \eqref{eq:midiscxx} since Definition \ref{def:micontinuous} does not cover ${\rm MI}(\boldsymbol{X},\boldsymbol{X})$, again because the pair $(\boldsymbol{X},\boldsymbol{X})$ is not absolutely continuous. Additionally, restatements of \eqref{eq:micondproperty2} and \eqref{eq:infodonthurt} for differential entropy also hold.

On the whole, MI for absolutely continuous random vectors verifies most important properties from the discrete case, including being symmetric and non-negative. Moreover, the value $0$ is obtained if and only if the random variables are independent. Concerning a parallel of \eqref{eq:miupperbound} for absolutely continuous random vectors, there is no natural finite upper bound for $h(\boldsymbol{X})$ in the continuous case. In fact, while the expression ${\rm MI}(\boldsymbol{X},\boldsymbol{Y})= h(\boldsymbol{X})-h(\boldsymbol{X}|\boldsymbol{Y})$, similar to \eqref{eq:micondproperty}, holds, $h(\boldsymbol{X}|\boldsymbol{Y})$ and $h(\boldsymbol{Y}|\boldsymbol{X})$ are not necessarily non-negative. Furthermore, as noted in Example \ref{ex:normal}, differential entropies can be become arbitrarily small, which applies, in particular, to the terms $h(\boldsymbol{X}|\boldsymbol{Y})$ and $h(\boldsymbol{Y}|\boldsymbol{X})$. As a result, ${\rm MI}(\boldsymbol{X},\boldsymbol{Y})$ can grow arbitrarily.

\subsubsection{Combination of continuous with discrete random vectors}

The definition of MI when we have an absolutely continuous random vector and a discrete random vector is also important in later stages of this article. For this reason, and despite the fact that the results that follow are naturally obtained from those that involve only either discrete or absolutely continuous vectors, we briefly go through them now.

\begin{definition}\label{def:micontdisc}
The MI between an absolutely continuous random vector $\boldsymbol{X}$ and a discrete random vector $\boldsymbol{Y}$ is given by either of the following two expressions:
\begin{align*}
{\rm MI}(\boldsymbol{X},\boldsymbol{Y})&=\sum_{\boldsymbol{y} \in \mathcal{Y}} P(\boldsymbol{Y}=\boldsymbol{y}) \int_{\boldsymbol{x}\in \mathcal{X}} f_{\boldsymbol{X}|\boldsymbol{Y}=\boldsymbol{y}}(\boldsymbol{x}) \ln \frac{f_{\boldsymbol{X}|\boldsymbol{Y}=\boldsymbol{y}}(\boldsymbol{x})}{f_{\boldsymbol{X}}(\boldsymbol{x})} d\boldsymbol{x} 
\\
&=
\int_{\boldsymbol{x}\in \mathcal{X}} f_{\boldsymbol{X}}(\boldsymbol{x}) \sum_{\boldsymbol{y} \in \mathcal{Y}} P(\boldsymbol{Y}=\boldsymbol{y}|\boldsymbol{X}=\boldsymbol{x}) \ln \frac{P(\boldsymbol{Y}=\boldsymbol{y}|\boldsymbol{X}=\boldsymbol{x})}{P(\boldsymbol{Y}=\boldsymbol{y})} d\boldsymbol{x}.
\end{align*}
\end{definition}

The majority of the properties stated for the discrete case are still valid in this case. In particular, analogues of \eqref{eq:micondproperty} hold, both in terms of entropies as well as in terms of differential entropies:
\begin{align}
\label{eq:micondpropertymix}
{\rm MI}(\boldsymbol{X},\boldsymbol{Y})&= h(\boldsymbol{X})-h(\boldsymbol{X}|\boldsymbol{Y})\\
\label{eq:micondpropertymix2}
&= H(\boldsymbol{Y})-H(\boldsymbol{Y}|\boldsymbol{X}).
\end{align}
Furthermore, ${\rm MI}(\boldsymbol{X},\boldsymbol{Y})\leq H(\boldsymbol{Y})$ is the analogue of \eqref{eq:miupperbound} for this setting. Note that \eqref{eq:micondpropertymix2}, but not \eqref{eq:micondpropertymix}, can be used to obtain an upper bound for ${\rm MI}(\boldsymbol{X},\boldsymbol{Y})$ since $h(\boldsymbol{X}|\boldsymbol{Y})$ may be negative.

%% file: section3.tex
\section{Triple mutual information and conditional mutual information}
\label{sec:tmicmi}

In this section, we discuss definitions and important properties associated with conditional MI and MI between three random vectors. Random vectors are considered to be discrete in this section as the generalization of the results for absolutely continuous random vectors would follow a similar approach.

\subsection{Conditional mutual information}
\label{subsec:conditionalmi}

\textit{Conditional MI} is defined in terms of entropies as follows, in a similar way to property \eqref{eq:micondproperty} \citep[cf.][]{MR2248026,meyer2006use}. 

\begin{definition}
\label{def:miconddiscrete}
The conditional MI between two random vectors $\boldsymbol{X}$ and $\boldsymbol{Y}$ given the random vector $\boldsymbol{Z}$ is written as
\begin{equation}
{\rm MI}(\boldsymbol{X},\boldsymbol{Y}|\boldsymbol{Z})= H(\boldsymbol{X}|\boldsymbol{Z})-H(\boldsymbol{X}|\boldsymbol{Y},\boldsymbol{Z}).
\label{eq:miconddiscrete}
\end{equation}
\end{definition}

Using \eqref{eq:miconddiscrete} and an analogue of the chain rule for conditional entropy, we conclude that:
\begin{equation}
{\rm MI}(\boldsymbol{X},\boldsymbol{Y}|\boldsymbol{Z})=H(\boldsymbol{X}|\boldsymbol{Z})+H(\boldsymbol{Y}|\boldsymbol{Z})-H(\boldsymbol{X},\boldsymbol{Y}|\boldsymbol{Z}).
\label{eq:miconddiscrete3}
\end{equation}

In view of Definition \ref{def:miconddiscrete}, developing the involved terms according to Definition \ref{def:midiscrete}, we obtain:
\begin{equation}
{\rm MI}(\boldsymbol{X},\boldsymbol{Y}|\boldsymbol{Z})= E_{\boldsymbol{Z}}[{\rm MI}(\tilde{\boldsymbol{X}}(\boldsymbol{Z}),\tilde{\boldsymbol{Y}}(\boldsymbol{Z})],
\label{eq:miconddiscreteexpvalue}
\end{equation}
where, for $\boldsymbol{z}\in \mathcal{Z}$, $(\tilde{\boldsymbol{X}}(\boldsymbol{z}),\tilde{\boldsymbol{Y}}(\boldsymbol{z}))$ is equal in distribution to $(\boldsymbol{X},\boldsymbol{Y})|\boldsymbol{Z}=\boldsymbol{z}$. 

Taking \eqref{eq:midiscnonnegativity} and \eqref{eq:miconddiscreteexpvalue} into account, 
\begin{equation}
{\rm MI}(\boldsymbol{X},\boldsymbol{Y}|\boldsymbol{Z})\geq 0,
\label{eq:micondnoneg}
\end{equation}
and ${\rm MI}(\boldsymbol{X},\boldsymbol{Y}|\boldsymbol{Z})= 0$ if and only if $\boldsymbol{X}$ and $\boldsymbol{Y}$ are conditionally independent given $\boldsymbol{Z}$.

Moreover, from \eqref{eq:miconddiscrete} and \eqref{eq:micondnoneg}, we conclude the following result similar to \eqref{eq:infodonthurt}:
\begin{equation}
H(\boldsymbol{X}|\boldsymbol{Y},\boldsymbol{Z}) \leq H(\boldsymbol{X}|\boldsymbol{Z}).
\label{eq:infodonthurtgen}
\end{equation}

\subsection{Triple mutual information}
\label{subsec:triplemi}

The generalization of the concept of MI to more than two random vectors is not unique. One such definition, associated with the concept of total correlation, was proposed by \cite{watanabe1960information}. An alternative one, proposed by \cite{Bell02}, is called \textit{triple MI} (TMI). We will consider the latter since it is the most meaningful in the context of objective functions associated with the problem of forward feature selection.

\begin{definition}
\label{def:mi3discrete}
The triple MI between three random vectors $\boldsymbol{X}$, $\boldsymbol{Y}$, and $\boldsymbol{Z}$ is defined as
\begin{align*}
{\rm TMI}(\boldsymbol{X},\boldsymbol{Y},\boldsymbol{Z})&=\sum_{\boldsymbol{x} \in \mathcal{X}} \sum_{\boldsymbol{y} \in \mathcal{Y}} \sum_{\boldsymbol{z} \in \mathcal{Z}} P(\boldsymbol{X}=\boldsymbol{x},\boldsymbol{Y}=\boldsymbol{y},\boldsymbol{Z}=\boldsymbol{z})\times \\ 
	& \hspace*{-2.0cm} \ln\frac{P(\boldsymbol{X}=\boldsymbol{x},\boldsymbol{Y}=\boldsymbol{y})P(\boldsymbol{Y}=\boldsymbol{y},\boldsymbol{Z}=\boldsymbol{z})P(\boldsymbol{X}=\boldsymbol{x},\boldsymbol{Z}=\boldsymbol{z})}{P(\boldsymbol{X}=\boldsymbol{x},\boldsymbol{Y}=\boldsymbol{y},\boldsymbol{Z}=\boldsymbol{z})P(\boldsymbol{X}=\boldsymbol{x})P(\boldsymbol{Y}=\boldsymbol{y})P(\boldsymbol{Z}=\boldsymbol{z})}.
\end{align*}
\end{definition}

Using the definition of MI and TMI, we can conclude that TMI and conditional MI are related in the following way, which provides extra intuition about the two concepts:
\begin{equation}
{\rm TMI}(\boldsymbol{X},\boldsymbol{Y},\boldsymbol{Z})= {\rm MI}(\boldsymbol{X},\boldsymbol{Y}) - {\rm MI}(\boldsymbol{X},\boldsymbol{Y}|\boldsymbol{Z}).
\label{eq:mi3discrete2}
\end{equation}
The TMI is not necessarily non-negative. This fact is exemplified and discussed in detail in the next section.

%% file: section4.tex
\section{The forward feature selection problem}
\label{sec:problem}

In this section, we focus on explaining the general context concerning forward feature selection methods based on mutual information. We first introduce target objective functions to be maximized in each step; we then define important concepts and prove some properties of such target objective functions. In the rest of this section, features are considered to be discrete for simplicity. The name target objective functions comes from the fact that, as we will argue, these are objective functions that perform exactly as we would desire ideally, so that a good method should reproduce its properties as well as possible.

\subsection{Target objective functions}
\label{subsec:idealof}

Let $C$ represent the class, which identifies the group each object belongs to. $\boldsymbol{S}$ ($\boldsymbol{F}$), in turn, denote the set of selected (unselected) features at a certain step of the iterative algorithm; in fact, $\boldsymbol{S} \cap \boldsymbol{F}=\emptyset$, and $\boldsymbol{S} \cup \boldsymbol{F}$ is the set with all input features. In what follows, when a set of random variables is in the argument of an entropy or MI term, it stands for the random vector composed by the random variables it contains.

Given the set of selected features, forward feature selection methods aim to select a candidate feature $X_j \in \boldsymbol{F}$ such that
\begin{equation*}
X_j=\arg{\max_{X_i \in \boldsymbol{F}}{{\rm  MI}(C,\boldsymbol{S} \cup \{X_i\})} }.
\end{equation*} 
Therefore, $X_j$ is, among the features in $\boldsymbol{F}$, the feature $X_i$ for which $\boldsymbol{S}\cup \{X_i\}$ maximizes the association (measured using MI) with the class, $C$. Note that we choose the feature that maximizes ${\rm MI}(C,X_i)$ in the first step (i.e., when $\boldsymbol{S}=\emptyset$).

Since ${\rm MI}(C,\boldsymbol{S} \cup \{X_i\}) = {\rm MI}(C,\boldsymbol{S}) + {\rm MI}(C,X_i|\boldsymbol{S})$ \citep[cf.][]{MR2422423}, in view of \eqref{eq:mi3discrete2}, the objective function evaluated at the candidate feature $X_i$ can be written as 
\begin{align}
\nonumber
 {\rm OF}(X_i)&={\rm MI}(C,\boldsymbol{S}) + {\rm MI}(C,X_i|\boldsymbol{S})\\
\nonumber
 &={\rm MI}(C,\boldsymbol{S}) + {\rm MI}(C,X_i) - {\rm TMI}(C,X_i,\boldsymbol{S})\\
\label{eq:maximization}
 &={\rm MI}(C,\boldsymbol{S}) + {\rm MI}(C,X_i) - {\rm MI}(X_i,\boldsymbol{S}) + {\rm MI}(X_i,\boldsymbol{S}|C).
\end{align}

The feature selection methods try to approximate this objective function. However, since the term ${\rm MI}(C,\boldsymbol{S})$ does not depend on $X_i$, most approximations can be studied taking as a reference the simplified form of objective function given by
\begin{equation*}
	{\rm OF'}(X_i)={\rm MI}(C,X_i) - {\rm MI}(X_i,\boldsymbol{S}) + {\rm MI}(X_i,\boldsymbol{S}|C).
\end{equation*}
This objective function has distinct properties from those of \eqref{eq:maximization} and, therefore, deserves being addressed separately. Moreover, it is the reference objective function for most feature selection methods.

The objective functions ${\rm OF}$ and ${\rm OF'}$ can be written in terms of entropies, which provides a useful interpretation. Using \eqref{eq:micondproperty}, we obtain for the first objective function:
\begin{equation}
{\rm OF}(X_i)=H(C)-H(C|X_i,\boldsymbol{S}).
\label{eq:maximization_entropy}
\end{equation}
Maximizing $H(C)-H(C|X_i,\boldsymbol{S})$ provides the same candidate feature $X_j$ as minimizing $H(C|X_i,\boldsymbol{S})$, for $X_i\in \boldsymbol{F}$. This means that the feature to be selected is the one leading to the minimal uncertainty of the class among the candidate features. As for the second objective function, we obtain, using again \eqref{eq:micondproperty}:
\begin{equation}
{\rm OF'}(X_i)=H(C|\boldsymbol{S})-H(C|X_i,\boldsymbol{S}).
\label{eq:OF_global_entropy}
\end{equation}
This emphasizes that a feature that maximizes \eqref{eq:maximization_entropy} also maximizes \eqref{eq:OF_global_entropy}. In fact, the term that depends on $X_i$ is the same in the two expressions.

We now provide bounds for the target objective functions.

\begin{theorem}
Given a general candidate feature $X_i$: 
\begin{enumerate}
\item $H(C)-H(C|\boldsymbol{S}) \leq {\rm OF}(X_i) \leq H(C)$.
\item $0 \leq {\rm OF'}(X_i) \leq H(C|\boldsymbol{S})$.
\end{enumerate}
\label{th:boundsgeneral}
\end{theorem}
\begin{proof}
Using the corresponding representations \eqref{eq:maximization_entropy} and \eqref{eq:OF_global_entropy} of the associated objective functions, the upper bounds follow from $H(C|X_i,\boldsymbol{S})\geq 0$. As for the lower bounds, in the case of statement 1, it comes directly from the fact that ${\rm OF'}(X_i)={\rm MI}(C,X_i|\boldsymbol{S})\geq 0$. As for statement 2, given that, from \eqref{eq:maximization_entropy}, ${\rm OF}(X_i)=H(C)-H(C|X_i,\boldsymbol{S})=H(C)-H(C|\boldsymbol{S})+{\rm MI}(C,X_i|\boldsymbol{S})$, we again only need to use the fact that ${\rm MI}(C,X_i|\boldsymbol{S})\geq 0$.
\end{proof}

The upper bound for ${\rm OF}$, $H(C)$, corresponds to the uncertainty in $C$, and the upper bound on ${\rm OF'}$, $H(C|\boldsymbol{S})$, corresponds to the uncertainty in $C$ not explained by the already selected features, $\boldsymbol{S}$. This is coherent with the fact that ${\rm OF'}$ ignores the term ${\rm MI}(C,\boldsymbol{S})$. The lower bound for ${\rm OF}$ corresponds to the uncertainty in $C$ already explained by $\boldsymbol{S}$.

\subsection{Feature types and their properties}
\label{subsec:localdefs}

Features can be characterized according to their usefulness in explaining the class at a particular step of the feature selection process. There are two broad types of features, those that add information to the explanation of the class, i.e. for which ${\rm MI}(C,X_i|\boldsymbol{S})>0$, and those that do not, i.e. for which ${\rm MI}(C,X_i|\boldsymbol{S})=0$. However, a finer categorization is needed to fully determine how the feature selection process should behave. We define four types of features: irrelevant, redundant, relevant, and fully relevant.

\begin{definition}
Given a subset of already selected features, $\boldsymbol{S}$, at a certain step of a forward sequential method, where the class is $C$, and a candidate feature $X_i$, then:
\begin{itemize}
\item $X_i$ is irrelevant given $(C,\boldsymbol{S})$ if ${\rm MI}(C,X_i|\boldsymbol{S})=0\, \wedge \, H(X_i|\boldsymbol{S})>0$;
\item $X_i$ is redundant given $\boldsymbol{S}$ if $H(X_i|\boldsymbol{S})=0$;
\item $X_i$ is relevant given $(C,\boldsymbol{S})$ if ${\rm MI}(C,X_i|\boldsymbol{S})>0$;
\item $X_i$ is fully relevant given $(C,\boldsymbol{S})$ if $H(C|X_i,\boldsymbol{S})=0\, \wedge \, H(C| \boldsymbol{S})>0$.
\end{itemize}
If $\boldsymbol{S}=\emptyset$, then ${\rm MI}(C,X_i|\boldsymbol{S})$, $H(X_i|\boldsymbol{S})$, $H(C|\boldsymbol{S})$, and $H(C|X_i,\boldsymbol{S})$ should be replaced by ${\rm MI}(C,X_i)$, $H(X_i)$, $H(C)$, and $H(C|X_i)$, respectively.
\label{def:subset}
\end{definition}

Under this definition, irrelevant, redundant, and relevant features form a partition of the set of candidate features $\boldsymbol{F}$. Note that fully relevant features are also relevant since $H(C|X_i, \boldsymbol{S})=0$ and $H(C|\boldsymbol{S})>0$ imply that ${\rm MI}(C,X_i|\boldsymbol{S})=H(C|\boldsymbol{S})-H(C|X_i,\boldsymbol{S})>0$.

Our definition introduces two novelties regarding previous works: first, we separate non-relevant features in two categories, of irrelevant and redundant features; second, we introduce the important category of fully relevant features.

Our motivation for separating irrelevant from redundant features is that, while a redundant feature remains redundant at all subsequent steps of the feature selection process, the same does not hold necessarily for irrelevant features. The following example illustrates how an irrelevant feature can later become relevant.

\begin{example}
We consider a class $C=(X+Y)^2$ where $X$ and $Y$ are two independent candidate features that follow uniform distributions on $\{-1,1\}$. $C$ follows a uniform distribution on $\{0,4\}$ and, as a result, the entropies of $X$, $Y$ and $C$ are $\ln(2)$. It can be easily checked that both $X$ and $Y$ are independent of the class. In the feature selection process, both features are initially irrelevant since, due to their independence from $C$, ${\rm MI}(C,X)={\rm MI}(C,Y)=0$. Suppose that $X$ is selected first. Then, $Y$ becomes relevant since ${\rm MI}(C,Y|X)=\ln(2)>0$, and it is even fully relevant since $H(C|Y,X)=0$ and $H(C|X)=\ln(2)>0$.
	\label{ex:complementarity}
\end{example}

The following theorem shows that redundant features always remain redundant.

\begin{theorem}
If a feature is redundant given $\boldsymbol{S}$, then it is also redundant given $\boldsymbol{S}'$, for $\boldsymbol{S}\subset \boldsymbol{S}'$.
\label{th:redundancy}
\end{theorem}
\begin{proof}
Suppose that $X_i$ is a redundant feature given $\boldsymbol{S}$, so that $H(X_i|\boldsymbol{S})=0$, and $\boldsymbol{S}\subset \boldsymbol{S}'$. This implies that $H(X_i|\boldsymbol{S}')=0$ by \eqref{eq:infodonthurtgen}. As a result, $X_i$ is also redundant given $\boldsymbol{S}'$.
\end{proof}

This result has an important practical consequence: features that are found redundant at a certain step of the feature selection process can be immediately removed from the set of candidate features $\boldsymbol{F}$, alleviating in this way the computational effort associated with the feature selection process.

Regarding relevant features, note that there are several levels of relevancy, as measured by ${\rm MI}(C,X_i|\boldsymbol{S})$. Fully relevant features form an important subgroup of relevant features since, together with already selected features, they completely explain the class, i.e. $H(C|\boldsymbol{S})$ becomes $0$ after selecting a fully relevant feature. Thus, all remaining unselected features are necessarily either irrelevant or redundant and the algorithm must stop. This also means that detecting a fully relevant feature can be used as a stopping criterion of forward feature selection methods. The condition $H(C| \boldsymbol{S})>0$ in the definition of fully relevant feature is required since an unselected feature can no longer be considered of this type after $H(C| \boldsymbol{S})=0$.

A stronger condition that could be considered as a stopping criterion is $H(C|\boldsymbol{S})=H(C|\boldsymbol{S},\boldsymbol{F})$, meaning that the (complete) set of candidate features $\boldsymbol{F}$ has no further information to explain the class. As in the previous case, the candidate features will all be irrelevant or redundant. However, since forward feature selection algorithms only consider one candidate feature at each iteration, and the previous condition requires considering all candidate features simultaneously, such condition cannot be used as a stopping criterion.

Regarding the categorization of features introduced by other authors, \cite{Brown:2012:CLM:2188385.2188387} considered only one category of non-relevant features, named irrelevant, consisting of the candidate features $X_i$ such that ${\rm MI}(C,X_i|\boldsymbol{S})=0$. \cite{meyer2008information} and \cite{vergara2014review} considered both irrelevant and redundant features. The definition of irrelevant feature is the one of \cite{Brown:2012:CLM:2188385.2188387}; redundant features are defined as features such that $H(X_i|\boldsymbol{S})=0$. Since the latter condition implies that ${\rm MI}(C,X_i|\boldsymbol{S})=0$ by \eqref{eq:micondproperty} and \eqref{eq:infodonthurtgen}, it turns out that redundant features are only a special case of irrelevant ones, which is not in agreement with our definition.

According to the feature types introduced above, a good feature selection method must select, at a given step, a relevant feature, preferably a fully relevant one, keep irrelevant features for future consideration and discard redundant features. The following theorem relates these desirable properties with the values taken by the target objective functions.

\begin{theorem}
\hfill\break
\vspace{-3ex}
\begin{enumerate}
\item If $X_i$ is a fully relevant feature given $(C,\boldsymbol{S})$, then ${\rm OF}(X_i) = H(C)$ and ${\rm OF'}(X_i) = H(C|\boldsymbol{S})$, i.e., the maximum possible values taken by the target objective functions are reached; recall Theorem \ref{th:boundsgeneral}. 
\item If $X_i$ is an irrelevant feature given $(C,\boldsymbol{S})$, then ${\rm OF}(X_i)=H(C)-H(C|\boldsymbol{S})$ and ${\rm OF'}(X_i)=0$, i.e., the minimum possible values of the target objective functions are reached; recall Theorem \ref{th:boundsgeneral}. 
\item If $X_i$ is a redundant feature given $\boldsymbol{S}$, then ${\rm OF}(X_i)=H(C)-H(C|\boldsymbol{S})$ and ${\rm OF'}(X_i)=0$, i.e., the minimum possible values of the target objective functions are reached; recall Theorem \ref{th:boundsgeneral}.
\item If $X_i$ is a relevant feature, but not fully relevant, given $(C,\boldsymbol{S})$, then $H(C)-H(C|\boldsymbol{S})<{\rm OF}(X_i)<H(C)$ and $0<{\rm OF'}(X_i)<H(C|\boldsymbol{S})$.
\end{enumerate}
\label{th:propsstandard}
\end{theorem}
\begin{proof}
The two equalities in statement 1 are an immediate consequence of equations \eqref{eq:maximization_entropy} and \eqref{eq:OF_global_entropy}, using the fact that $H(C|X_i,\boldsymbol{S})=0$ if $X_i$ is fully relevant given $(C,\boldsymbol{S})$.

Suppose that $X_i$ is an irrelevant feature given $(C,\boldsymbol{S})$, so that ${\rm MI}(C,X_i|\boldsymbol{S})=0$. Then, the relation ${\rm OF'}(X_i)=0$ results directly from ${\rm OF'}(X_i)={\rm MI}(C,X_i|\boldsymbol{S})$. Conversely, the relation ${\rm OF}(X_i)=H(C)-H(C|\boldsymbol{S})$ follows from the fact that ${\rm OF}(X_i)=H(C)-H(C|\boldsymbol{S})+{\rm MI}(C,X_i|\boldsymbol{S})$. As a result, statement 2 is verified.

The equalities in statement 3 follow likewise since ${\rm MI}(C,X_i|\boldsymbol{S})=0$ if $X_i$ is a redundant feature given $\boldsymbol{S}$. 

As for statement 4, we need to prove that the objective functions neither take the minimum nor the maximum value for a relevant feature that is not fully relevant. We start by checking that the minimum values are not reached. The proof is similar to that of statement 2. Since ${\rm OF'}(X_i)={\rm MI}(C,X_i|\boldsymbol{S})$, and since the assumption is that ${\rm MI}(C,X_i|\boldsymbol{S})>0$, then ${\rm OF'}(X_i)$ is surely larger than $0$. Concerning ${\rm OF}(X_i)$, since ${\rm OF}(X_i)=H(C)-H(C|\boldsymbol{S})+{\rm MI}(C,X_i|\boldsymbol{S})$ and ${\rm MI}(C,X_i|\boldsymbol{S})>0$, ${\rm OF}(X_i)$ must be larger than $H(C)-H(C|\boldsymbol{S})$. Concerning the upper bounds, the proof is now similar to that of statement 1. If the feature $X_i$ is not fully relevant given $(C,\boldsymbol{S})$, meaning that $H(C|\boldsymbol{S},X_i)>0$, the desired conclusions immediately follow from \eqref{eq:maximization_entropy} and \eqref{eq:OF_global_entropy}.
\end{proof}

Thus, fully relevant (irrelevant and redundant) features achieve the maximum (minimum) of the objective functions, and relevant features that are not fully relevant achieve a value between the maximum and the minimum values of the objective functions. These properties assure that the ordering of features at a given step of the feature selection process is always correct. Note that irrelevant and redundant features can be discriminated by evaluating $H(X_i|\boldsymbol{S})$.

\subsection{Complementarity}
\label{subsec:complementarity}

The concept of complementarity is associated with the TMI term of the target objective function, given by ${\rm TMI}(C,X_i,\boldsymbol{S})={\rm MI}(X_i,\boldsymbol{S}) - {\rm MI}(X_i,\boldsymbol{S}|C)$; recall \eqref{eq:mi3discrete2}. Following \cite{meyer2006use}, we say that $X_i$ and $\boldsymbol{S}$ are \textit{complementary} with respect to $C$ if $-{\rm TMI}(C,X_i,\boldsymbol{S})>0$. Interestingly, \cite{cheng2011conditional} refer to complementarity as the existence of \textit{positive interaction}, or \textit{synergy}, between $X_i$ and $\boldsymbol{S}$ with respect to $C$.

Given that ${\rm MI}(X_i,\boldsymbol{S}) \geq 0$, a negative TMI is necessarily associated with a positive value of ${\rm MI}(X_i,\boldsymbol{S}|C)$. This term expresses the contribution of a candidate feature to the explanation of the class, when taken together with already selected features. Following \cite{LinT06} and \cite{vinh2015can}, we call this term \textit{class-relevant redundancy}. \cite{Brown:2012:CLM:2188385.2188387} calls this term \textit{conditional redundancy}. Class-relevant redundancy is sometimes coined as the \textit{good} redundancy since it expresses an association that contributes to the explanation of the class. \cite{guyon2008feature} highlights that ``correlation does not imply redundancy'' to stress that association between $X_i$ and $\boldsymbol{S}$ is not necessarily bad.

The remaining term of the decomposition of TMI, ${\rm MI}(X_i,\boldsymbol{S})$, measures the association between the candidate feature and the already selected features. Following \cite{LinT06}, we call this term \textit{inter-feature redundancy}. It is sometimes coined as the \textit{bad} redundancy since it expresses the information of the candidate feature already contained in the set of already selected features.

Note that TMI takes negative values whenever the class-relevant redundancy exceeds the inter-feature redundancy, i.e. ${\rm MI}(X_i,\boldsymbol{S}|C)>{\rm MI}(X_i,\boldsymbol{S})$. A candidate feature $X_i$ for which ${\rm TMI}(C,X_i,\boldsymbol{S})$ is negative is a relevant feature, i.e. ${\rm MI}(C,X_i|\boldsymbol{S})\geq 0$, since ${\rm MI}(C,X_i|\boldsymbol{S})={\rm MI}(C,X_i)-{\rm TMI}(C,X_i,\boldsymbol{S})$ by \eqref{eq:mi3discrete2}, and ${\rm MI}(C,X_i) \geq 0$. Thus, a candidate feature may be relevant even if it is strongly associated with the already selected features. Moreover, class-relevant redundancy may turn a feature that was initially irrelevant into a relevant feature, as illustrated in Example \ref{ex:complementarity}. In that example, the candidate feature $Y$ was independent of the already selected one, $X$, i.e. ${\rm MI}(X_i,\boldsymbol{S})={\rm MI}(Y,X)=0$, but $Y$ taken together with $X$ had a positive contribution to the explanation of the class (indeed it fully explained the class), since the class-relevant redundancy is positive, i.e. ${\rm MI}(X_i,\boldsymbol{S}|C)={\rm MI}(Y,X|C)=\ln(2)>0$.

Authors in \cite{meyer2006use} provided an interesting interpretation of complementarity, noting that
\begin{equation*}
-{\rm TMI}(C,X_i,\boldsymbol{S})={\rm MI}(\{X_i\}\cup \boldsymbol{S},C)-{\rm MI}(X_i,C)-{\rm MI}(\boldsymbol{S},C). 
\end{equation*}
Thus, if $-{\rm TMI}(C,X_i,\boldsymbol{S})>0$, then ${\rm MI}(\{X_i\}\cup \boldsymbol{S},C)>{\rm MI}(X_i,C)+{\rm MI}(\boldsymbol{S},C)$. Therefore, $-{\rm TMI}(C,X_i,\boldsymbol{S})>0$ measures the gain resulting from considering $X_i$ and $\boldsymbol{S}$ together, instead of considering them separately, when measuring the association with the class $C$.

%% file: section5_RO.tex
\section{Representative feature selection methods}
\label{sec:methods}

The target objective functions discussed in Section \ref{sec:problem} cannot be used in practice since they require the joint distribution of $(C,X_i,\boldsymbol{S})$, which is not known and has to be estimated. This becomes more and more difficult as the cardinality of $\boldsymbol{S}$, denoted by $|\boldsymbol{S}|$ from here on, increases.

The common solution is to use approximations, leading to different feature selection methods. For the analysis in this paper, we selected a set of methods representative of the main types of approximations to the target objective functions. In what follows, we first describe the representative methods, and discuss drawbacks resulting from their underlying approximations; we then discuss how these methods cope with the desirable properties given by Theorem \ref{th:boundsgeneral} and Theorem \ref{th:propsstandard}; finally, we briefly refer to other methods proposed in the literature and how they relate to the representative ones. In this section, features are considered to be discrete for simplicity.

\subsection{Methods and their drawbacks}
\label{subsec:representative}

The methods selected to represent the main types of approximations to the target objective functions are: MIM \citep{lewis1992feature}, MIFS \citep{Battiti94usingmutual}, mRMR \citep{Peng05featureselection}, maxMIFS \citep{claudiapaper}, CIFE \citep{LinT06}, JMI \citep{YangM99}, CMIM \citep{MR2248026}, and JMIM \citep{bennasar2015feature}. These methods are listed in Table \ref{tab:varselmethods}, together with their objective functions. Note that, for all methods, including mRMR and JMI, the objective function in the first step of the algorithm is simply ${\rm MI}(C,X_i)$. This implies, in particular, that the first feature to be selected is the same in all methods.

\begin{table}[t!]
\centering
\caption{Objective functions of the representative feature selection methods, evaluated at candidate feature $X_i$.}
\small
\begin{tabular}{cl}
\toprule
Method & Objective function evaluated at $X_i$\\ 
\midrule
MIM & ${\rm MI}(C,X_i)$ \\ [0.2cm]
MIFS & $ {\rm MI}(C,X_i)-\beta \sum_{X_s \in \boldsymbol{S}} {\rm MI}(X_i,X_s)   $ \\ [0.2cm]
mRMR  & $ {\rm MI}(C,X_i)-\frac{1}{|\boldsymbol{S}|} \sum_{X_s \in \boldsymbol{S}} {\rm MI}(X_i,X_s) $\\ [0.2cm]
maxMIFS & $ {\rm MI}(C,X_i)-\max_{X_s \in \boldsymbol{S}}{\rm MI}(X_i,X_s)  $\\ [0.2cm]
CIFE  & $ {\rm MI}(C,X_i)-\sum_{X_s \in \boldsymbol{S}} \left({\rm MI}(X_i,X_s) - {\rm MI}(X_i,X_s|C) \right)$ \\ [0.2cm]
JMI  & $  {\rm MI}(C,X_i)-\frac{1}{|\boldsymbol{S}|} \sum_{X_s \in \boldsymbol{S}} \left({\rm MI}(X_i,X_s)- {\rm MI}(X_i,X_s|C)\right) $\\ [0.2cm]
CMIM  & $ {\rm MI}(C,X_i)-\max_{X_s \in \boldsymbol{S}}\left\{{\rm MI}(X_i,X_s)- {\rm MI}(X_i,X_s|C)\right\}  $\\[0.2cm] 
JMIM  & $ {\rm MI}(C,X_i)-\max_{X_s \in \boldsymbol{S}}\left\{{\rm MI}(X_i,X_s)- {\rm MI}(X_i,X_s|C)-{\rm MI}(C,X_s)\right\}$\\ 
\bottomrule
\end{tabular}
\label{tab:varselmethods}
\end{table}

The methods differ in the way their objective functions approximate the target objective functions. All methods except JMIM have objective functions that can be seen as approximations of the target ${\rm OF'}$; the objective function of JMIM can be seen as an approximation of the target ${\rm OF}$. The approximations taken by the methods are essentially of three types: approximations that ignore both types of redundancy (inter-feature and class-relevant), that ignore class-relevant redundancy but consider an approximation for the inter-feature redundancy, and that consider an approximation for both the inter-feature and class-relevant redundancies. These approximations introduce drawbacks in the feature selection process with different degrees of severity, discussed next. The various drawbacks are summarized in Table \ref{tab:drawback}.

The simplest method is MIM. This method discards the TMI term of the target objective function ${\rm OF'}$, i.e.
\begin{equation}
{\rm OF'}(X_i) \approx {\rm MI}(C,X_i).
\label{eq:OF_MIM}
\end{equation}
Thus, MIM ranks features accounting only for relevance effects, and completely ignores redundancy. We call the drawback introduced by this approximation \textit{redundancy ignored}.

The methods MIFS, mRMR, and maxMIFS ignore complementarity effects, by approximating the TMI term of ${\rm OF'}$ through the inter-feature redundancy term only, i.e. by discarding the class-relevant redundancy. Thus,
\begin{equation}
{\rm OF'}(X_i)\approx {\rm MI}(C,X_i) - {\rm MI}(X_i,\boldsymbol{S}).
\label{eq:OF_global_bi}
\end{equation}
In this case, the TMI can no longer take negative values, since it reduces to the term ${\rm MI}(X_i,\boldsymbol{S})$. As discussed in Section \ref{subsec:complementarity}, the complementarity expresses the contribution of a candidate feature to the explanation of the class, when taken together with already selected features, and ignoring this contribution may lead to gross errors in the feature selection process. This drawback will be called \textit{complementarity ignored}, and it was noted by \cite{Brown:2012:CLM:2188385.2188387}. These methods include an additional approximation, to calculate the TMI term ${\rm MI}(X_i,\boldsymbol{S})$, which is also used by the methods that do not ignore complementarity, and will be discussed next.

The methods that do not ignore complementarity, i.e. CIFE, JMI, CMIM, and JMIM, approximate the terms of the objective functions that depend on the set $\boldsymbol{S}$, i.e. ${\rm MI}(C,\boldsymbol{S})$, ${\rm MI}(X_i,\boldsymbol{S})$, and ${\rm MI}(X_i,\boldsymbol{S}|C)$, which are difficult to estimate, through a function of the already selected features $X_s$, $X_s\in \boldsymbol{S}$, taken individually. Considering only individual associations neglects higher order associations, e.g. between a candidate and two or more already selected features. 
Specifically, for CIFE, JMI, and CMIM,
\begin{equation*}
{\rm OF'}(X_i)\approx {\rm MI}(C,X_i)- \Gamma({\rm TMI}(C,X_i,\boldsymbol{S}))
\end{equation*}
and for JMIM,
\begin{equation*}
{\rm OF}(X_i)\approx {\rm MI}(C,X_i)- \Gamma({\rm TMI}(C,X_i,\boldsymbol{S})-{\rm MI}(C,\boldsymbol{S})),
\end{equation*}
where $\Gamma$ denotes an approximating function. This type of approximation is also used by the methods that ignore complementarity. Hereafter, we denote an already selected feature $X_s\in \boldsymbol{S}$ simply by $X_s$. Three types of approximating functions have been used: a sum of $X_s$ terms scaled by a constant (MIFS and CIFE), an average of $X_s$ terms (mRMR and JMI), and a maximization over $X_s$ terms (maxMIFS, CMIM, and JMIM).

MIFS and CIFE approximate the TMI by a sum of $X_s$ terms scaled by a constant. In particular, for CIFE,
\begin{align*}
{\rm TMI}(C,X_i,\boldsymbol{S}) \approx & \sum_{X_s\in\boldsymbol{S}} {\rm TMI}(C,X_i,X_s) = \sum_{X_s \in \boldsymbol{S}} \left[{\rm MI}(X_i,X_s) - {\rm MI}(X_i,X_s|C)\right]\\
= &\sum_{X_s \in \boldsymbol{S}} {\rm MI}(X_i,X_s) - \sum_{X_s \in \boldsymbol{S}} {\rm MI}(X_i,X_s|C).
\end{align*}
The MIFS approximation is similar, but without the class-relevant redundancy terms, and with the sum of inter-feature redundancy terms scaled by a constant $\beta$. In both cases, a problem arises because the TMI is approximated by a sum of terms which individually have the same scale as the term they try to approximate. This results in an approximation of the TMI that can have a much larger scale than the original term. Since these terms are both redundancy terms, we will refer to this as the \textit{redundancy overscaled} drawback. It becomes more and more severe as $\boldsymbol{S}$ grows. This drawback was also noted by \cite{Brown:2012:CLM:2188385.2188387}, referring to it as the problem of not balancing the magnitudes of the relevancy and the redundancy.

Two other approximating functions were introduced to overcome the redundancy overscaled drawback. The first function, used by mRMR and JMI, replaces the TMI by an average of $X_s$ terms. In particular, for JMI,
\begin{align*}
{\rm TMI}(C,X_i,\boldsymbol{S})\approx & \frac{1}{|\boldsymbol{S}|}\sum_{X_s\in\boldsymbol{S}} {\rm TMI}(C,X_i,X_s) = \frac{1}{|\boldsymbol{S}|} \sum_{X_s \in \boldsymbol{S}} \left[{\rm MI}(X_i,X_s) - {\rm MI}(X_i,X_s|C)\right]\\
= &\frac{1}{|\boldsymbol{S}|} \sum_{X_s \in \boldsymbol{S}} {\rm MI}(X_i,X_s) - \frac{1}{|\boldsymbol{S}|} \sum_{X_s \in \boldsymbol{S}} {\rm MI}(X_i,X_s|C).
\end{align*}
The mRMR approximation is similar, but without the class-relevant redundancy terms. This approximation solves the overscaling problem but introduces another drawback. In fact, since ${\rm MI}(X_i,\boldsymbol{S}) \geq {\rm MI}(X_i,X_s)$, implying that ${\rm MI}(X_i,\boldsymbol{S}) \geq \frac{1}{|\boldsymbol{S}|}\sum_{X_s\in\boldsymbol{S}} {\rm MI}(X_i,X_s)$, the approximation undervalues the inter-feature redundancy; at the same time, given that ${\rm MI}(X_i,\boldsymbol{S}|C) \geq {\rm MI}(X_i,X_s|C)$, implying ${\rm MI}(X_i,\boldsymbol{S}|C)\geq \frac{1}{|\boldsymbol{S}|}\sum_{X_s\in\boldsymbol{S}} {\rm MI}(X_i,X_s|C)$, it also undervalues the class-relevant redundancy. We call this drawback \textit{redundancy undervalued}.

The second approximating function introduced to overcome the redundancy overscaled drawback is a maximization over $X_s$ terms. This approximation is used differently in maxMIFS and CMIM, on one side, and JMIM, on the other. Methods maxMIFS and CMIM just replace the TMI by a maximization over $X_s$ terms. In particular, for CMIM,
\begin{equation*}
{\rm TMI}(C,X_i,\boldsymbol{S})\approx \max_{X_s\in\boldsymbol{S}} {\rm TMI}(C,X_i,X_s) = \max_{X_s \in \boldsymbol{S}} \left({\rm MI}(X_i,X_s)- {\rm MI}(X_i,X_s|C)\right).
\end{equation*}
The maxMIFS approximation is similar, but without the class-relevant redundancy terms. 

The discussion regarding the quality of the approximation is more complex in this case. We start by maxMIFS. In this case, since ${\rm MI}(X_i,\boldsymbol{S}) \geq {\rm MI}(X_i,X_s)$,
\begin{equation}
{\rm MI}(X_i,\boldsymbol{S}) \geq \max_{X_s \in \boldsymbol{S}}{\rm MI}(X_i,X_s)\geq \frac{1}{|\boldsymbol{S}|} \sum_{X_s \in \boldsymbol{S}} {\rm MI}(X_i,X_s).
\label{eq:maxr}
\end{equation}
Thus, this approximation still undervalues inter-feature redundancy, but is clearly better than the one considering an average. Indeed, we may say that maximizing over $X_s$ is the best possible approximation, under the restriction that only one $X_s$ is considered.

Regarding CMIM, we first note that a relationship similar to \eqref{eq:maxr} also holds for the class-relevant redundancy, i.e.
\begin{equation*}
{\rm MI}(X_i,\boldsymbol{S}|C) \geq \max_{X_s \in \boldsymbol{S}}{\rm MI}(X_i,X_s|C)\geq \frac{1}{|\boldsymbol{S}|} \sum_{X_s \in \boldsymbol{S}} {\rm MI}(X_i,X_s|C),
\end{equation*}
since ${\rm MI}(X_i,\boldsymbol{S}|C) \geq {\rm MI}(X_i,X_s|C)$. However, while it is true for the two individual terms that compose the TMI that ${\rm MI}(X_i,\boldsymbol{S}) \geq \max_{X_s \in \boldsymbol{S}}{\rm MI}(X_i,X_s)$ and ${\rm MI}(X_i,\boldsymbol{S}|C) \geq \max_{X_s \in \boldsymbol{S}}{\rm MI}(X_i,X_s|C)$, it is no longer true that ${\rm TMI}(C,X_i,\boldsymbol{S}) \geq \max_{X_s \in \boldsymbol{S}}{\rm TMI}(C,X_i,X_s)$. Thus, the maximization over $X_s$ terms of CMIM is not as effective as that of maxMIFS. Moreover, applying a maximization jointly to the difference between the inter-feature and the class-relevant redundancy terms clearly favors $X_s$ features that together with $X_i$ have small class-relevant redundancy, i.e. a small value of ${\rm MI}(X_i,X_s|C)$. This goes against the initial purpose of methods that, like CMIM, introduced complementarity effects in forward feature selection methods. We call this drawback \textit{complementarity penalized}. We now give an example that illustrates how this drawback may impact the feature selection process.

\begin{example}
	Assume that we have the same features as in Example \ref{ex:complementarity}, plus two extra features $W$ and $Z$, independent of any vector containing other random variables of the set $\{W,Z,X,Y,C\}$. Moreover, consider the objective function of CMIM.
	
	In the first step, the objective function value is $0$ for all features. We assume that $W$ is selected first. In this case, at the second step, the objective functions value is again $0$ for all features. We assume that $X$ is selected. At the third step, $Y$ should be selected since it is fully relevant and $Z$ is irrelevant. At this step, the objective function value at $Z$ is 0. The objective function at $Y$ requires a closer attention. Since $Y$ is independent of the class, ${\rm MI}(Y,C)=0$, the target objective function evaluated at $Y$ is
	\begin{align*}
	-{\rm TMI}(C,Y,\{W,X\})&=-\left[{\rm MI}(Y,\{W,X\})-{\rm MI}(Y,\{W,X\}|C)\right]\\
	&=-\left[0-(H(Y|C)-H(Y|W,X,C))\right]=-(0-\ln(2))=\ln(2),
	\end{align*}
	and the objective function of CMIM evaluated at $Y$ is
	\begin{align*}
	&	-\max\{{\rm TMI}(C,Y,W),{\rm TMI}(C,Y,X)\}\\
	&\hspace*{1ex}=-\max\{{\rm MI}(Y,W)-{\rm MI}(Y,W|C),{\rm MI}(Y,X)-{\rm MI}(Y,X|C)\}\\
	&\hspace*{1ex}=-\max\{0-0,0-(H(Y|C)-H(Y|X,C))\}=-\max\{0-0,0-\ln(2)\}\\
	&\hspace*{1ex}=-\max\{0,-\ln(2)\}=0.
	\end{align*}
This shows that, according to CMIM, both $Y$ and $Z$ can be selected at this step, whereas $Y$ should be selected first, as confirmed by the target objective function values. The problem occurs because the class-relevant redundancy ${\rm MI}(Y,X|C)$ brings a negative contribution to the term of the maximization that involves $X$, leading to ${\rm TMI}(C,Y,X)=-ln(2)$, thus forcing the maximum to be associated with the competing term, since ${\rm TMI}(C,Y,W)=0$. As noted before, the maximum applied in this way penalizes the complementarity effects between $Y$ and $X$ that, as a result, are not reflected in the objective function of candidate $Y$; contrarily, the term that corresponds to an already selected feature that has no association with $Y$, i.e. the term involving $W$, is the one that is reflected in the objective function of candidate $Y$. 


\label{ex:problemCMIM}
\end{example}

Note that since ${\rm MI}(X_i,\boldsymbol{S}) \geq {\rm MI}(X_i,X_s)$ and ${\rm MI}(X_i,\boldsymbol{S}|C) \geq {\rm MI}(X_i,X_s|C)$, this approximation also undervalues both the inter-feature and the class-relevant redundancies. However, since the maximum is applied to the difference of the terms, it can no longer be concluded, as in the case of maxMIFS, that the approximation using a maximum is better than the one using an average (the case of JMI). In this case, the inter-feature redundancy term still pushes towards selecting the $X_s$ that leads to the maximum value of ${\rm MI}(X_i,X_s)$, since it contributes positively to the value inside the maximum operator; contrarily, the class-relevant redundancy term pushes towards selecting $X_s$ features that depart from the maximum value of ${\rm MI}(X_i,X_s|C)$, since it contributes negatively.

JMIM uses the approximation based on the maximization operator, like maxMIFS and CMIM. However, the maximization embraces an additional term. Specifically,
\begin{equation*}
{\rm TMI}(C,X_i,\boldsymbol{S})-{\rm MI}(C,\boldsymbol{S})\approx \max_{X_s\in \boldsymbol{S}}\{{\rm TMI}(C,X_i,X_s)-{\rm MI}(C,X_s)\}.
\end{equation*}
The additional term of JMIM, i.e. ${\rm MI}(C,X_s)$, tries to approximate a term of the target objective function that does not depend on $X_i$, i.e. ${\rm MI}(C,\boldsymbol{S})$, and brings additional problems to the selection process. We call this drawback \textit{unimportant term approximated}. JMIM inherits the drawbacks of CMIM, complementarity penalized and redundancy undervalued. Moreover, the extra term adds a negative contribution to each $X_s$ term, favoring $X_s$ features with small association with $C$, which goes against the whole purpose of the feature selection process.

The representations of the objective functions of CMIM and JMIM in the references where they were proposed \citep{MR2248026,bennasar2015feature} differ from the ones in Table \ref{tab:varselmethods}. More concretely, their objective functions were originally formalized in terms of minimum operators:
\begin{align}
	\label{eq:OF_CMIM}
	{\rm OF}_{{\rm CMIM}}(X_i)&= \min_{X_s\in \boldsymbol{S}} {\rm MI}(C,X_i|X_s);\\
	\label{eq:OF_JMIM}
	{\rm OF}_{{\rm JMIM}}(X_i)&= \min_{X_s \in \boldsymbol{S}} \left\{{\rm MI}(C,X_s)+ {\rm MI}(C,X_i|X_s)\right\}.
\end{align}
The representations in Table \ref{tab:varselmethods} result from the above ones using simple algebraic manipulation; recall \eqref{eq:mi3discrete2}. They allow a nicer and unified interpretation of the objective functions. For instance, they allow noticing much more clearly the similarities between maxMIFS and CMIM, as well as between CMIM and JMIM.

\begin{table}[t!]
\centering
\caption{Drawbacks of the representative feature selection methods.}
\scriptsize
\begin{tabular}{ccccccccc}
\toprule
{Drawback} & {MIM} & {MIFS} & {mRMR} & {maxMIFS} & {CIFE} & {JMI} & {CMIM} & {JMIM}\\
\midrule
\begin{tabular}{c} Redundancy\\ ignored\end{tabular} & X \\ 
\arrayrulecolor{black!30}\midrule
\begin{tabular}{c} Complementarity\\ ignored\end{tabular} & &X & X & X & & & & \\
\midrule
\begin{tabular}{c} Complementarity\\ penalized\end{tabular}& & & & & & & X & X \\
\midrule
\begin{tabular}{c} Redundancy\\ undervalued\end{tabular} & & & X & X & & X & X & X \\
\midrule
\begin{tabular}{c} Unimportant term\\ approximated \end{tabular}& & & & & & & & X \\
\midrule
\begin{tabular}{c} Redundancy\\ overscaled \end{tabular} & & X & & & X & & & \\
\arrayrulecolor{black}\hline
\end{tabular}
\label{tab:drawback}
\end{table}

\subsection{Properties of the methods}
\label{subsec:propsmethods}

The drawbacks presented in Section \ref{subsec:representative} have consequences in terms of the good properties that forward feature selection methods must have, as expressed by Theorem \ref{th:boundsgeneral} and Theorem \ref{th:propsstandard}: (i) the existence of meaningful bounds for the objective function, and (ii) the fact that fully relevant candidate features are the only ones that reach the maximum value of the objective function, while irrelevant and redundant features are the only ones to reach the minimum, which guarantees a perfect ordering of the features. With a few exceptions, the approximations taken by the various methods make them lose these properties.

Concerning the preservation of the bounds stated by Theorem \ref{th:boundsgeneral}, it can be shown that MIFS, mRMR, maxMIFS, and CIFE, do not preserve neither the lower bound nor the upper bound. Indeed, the objective function of CIFE is unbounded, both superiorly and inferiorly, due to the overscaled redundancy drawback. Moreover, the objective functions of MIFS, mRMR, and maxMIFS are unbounded inferiorly, due to the complementarity ignored drawback, i.e. due to the lack of the compensation provided by the class-relevant redundancy. For these methods, the upper bound of the objective function becomes $H(C)$. This bound is meaningless since it no longer expresses the uncertainty in $C$ not explained by already selected features.


MIM, JMI, CMIM, and JMIM preserve one of the bounds: JMIM preserves the upper bound and the remaining methods preserve the lower bound. MIM preserves the lower bound but just because its objective function is too simplistic.

In order to see that JMI preserves the lower bound, note that, using \eqref{eq:mi3discrete2}, its objective function can also be written as
\begin{equation}
{\rm OF}_{{\rm JMI}}(X_i)= \frac{1}{|\boldsymbol{S}|} \sum_{X_s\in \boldsymbol{S}} {\rm MI}(C,X_i|X_s).
\label{eq:OF_JMI}
\end{equation}
For any candidate feature $X_i$, since ${\rm MI}(C,X_i|X_s)\geq 0$ for all $X_s\in \boldsymbol{S}$, it follows that ${\rm OF}_{{\rm JMI}}(X_i)\geq 0$. Similarly, using \eqref{eq:OF_CMIM}, it follows immediately from the non-negativity of ${\rm MI}(C,X_i|X_s)$ that ${\rm OF}_{{\rm CMIM}}(X_i)\geq 0$, again for any candidate feature $X_i$.

To see that JMIM preserves the upper bound, note that, using \eqref{eq:OF_JMIM}, its objective function can also be written as
\begin{align*}
{\rm OF}_{{\rm JMIM}}(X_i)&= \min_{X_s \in \boldsymbol{S}} \left\{{\rm MI}(C,X_s)+ {\rm MI}(C,X_i|X_s)\right\}\\ 
&= \min_{X_s\in \boldsymbol{S}} \left\{H(C)-H(C|X_i,X_s)\right\}\\
&= H(C)-\max_{X_s\in S}H(C|X_i,X_s).
\end{align*}
Hence, the objective function of JMIM has $H(C)$ as upper bound for any candidate feature $X_i$ since $H(C|X_i,X_s)\geq 0$ for all $X_s\in \boldsymbol{S}$. This is the desired bound since this method has target objective function ${\rm OF}$ as reference.

Despite maintaining one of the bounds stated by Theorem \ref{th:boundsgeneral}, MIM, JMI, CMIM, and JMIM, do not preserve the bound that involves the conditional entropy $H(C|\boldsymbol{S})$. For MIM the upper bound becomes $H(C)$ which, as in the case of methods ignoring complementarity, is meaningless. For the remaining methods, the bound is lost due to the approximation that replaces the terms of the objective function that depend on set $\boldsymbol{S}$ by a function of the already selected features $X_s$ taken individually. As for the lower bound of JMIM and the upper bounds of JMI and CMIM, they are now functions of $H(C|X_s)$. As in the case of MIFS, mRMR, and maxMIFS, the new bounds become meaningless: the upper bounds of JMI and CMIM no longer express the uncertainty in $C$ that is not explained by the \textit{complete} set of already selected features; and the lower bound of JMIM no longer expresses the uncertainty in $C$ already explained by the \textit{complete} set of already selected features.

In Section \ref{sec:setting} we will illustrate the loss of bounds by the various methods.

Regarding the connections between the bounds of the objective functions and the feature types, stated by Theorem \ref{th:propsstandard}, these connections are lost for all methods, despite the fact that some bounds are preserved. It is no longer possible to assure that fully relevant features reach the maximum value of the objective function (when it exists) and that irrelevant and redundant features reach the minimum (when it exists). Moreover, the stopping criterion is lost. We will provide several examples in Section \ref{sec:setting}. This is again due to the approximation that replaces the dependencies on the whole set of already selected features, $\boldsymbol{S}$, by dependencies on individual features $X_s \in \boldsymbol{S}$, which is shared by all methods.

It would be useful to have results similar to those of Theorem \ref{th:propsstandard}, if their validity given $X_s$ would imply their validity given $\boldsymbol{S}$. This is only meaningful for feature types, such that if a feature has a type given $X_s$ it will have the same type given $\boldsymbol{S}$. Unfortunately, this is only true for redundant features. In fact, according to Theorem \ref{th:redundancy}, a feature that is redundant given $X_s$ will also be redundant given $\boldsymbol{S}$. The same does not hold for irrelevant and relevant features since, as discussed in Section \ref{subsec:localdefs}, as $\boldsymbol{S}$ grows, relevant features can become irrelevant, and vice-versa. Thus, only properties concerning redundancy given $X_s$ are worth being considered. In this respect, a weaker version of Theorem \ref{th:propsstandard}.3 can be proved for CMIM.

\begin{theorem}
If there exists $X_s\in \boldsymbol{S}$ such that $X_i$ is a redundant feature given $\{X_s\}$, then ${\rm OF}_{{\rm CMIM}}(X_i)=0$, i.e., the minimum possible value taken by the objective function of CMIM is reached. 
\label{th:ofzeroredundCMIM}
\end{theorem}
\begin{proof}
Since $X_i$ is a redundant feature given $\{X_s\}$, then ${\rm MI}(C,X_i|X_s)=0$ by \eqref{eq:micondproperty} and \eqref{eq:infodonthurtgen}. As a result, ${\rm OF}_{{\rm CMIM}}(X_i)=0$ follows from \eqref{eq:OF_CMIM}. In fact, in order for $\min_{X_s\in \boldsymbol{S}}{\rm MI}(C,X_i|X_s)$ to be 0, it is enough that ${\rm MI}(C,X_i|X_s)$ is 0 for one particular $X_s$, since the terms involved in the minimization are all non-negative. 
\end{proof}


Theorem \ref{th:ofzeroredundCMIM} states that the objective function of CMIM reaches the minimum for a feature that is redundant given $X_s$. Note that a feature can be redundant given $\boldsymbol{S}$ but not redundant given $X_s$, which renders this result weaker than that of Theorem \ref{th:propsstandard}.3. Theorems analogous to Theorem \ref{th:ofzeroredundCMIM} cannot be proved for the remaining methods, and we provide counter-examples in Section \ref{sec:setting}. In particular, the possibility to discard redundant features from the set of candidate features is lost, except for CMIM in the weaker context of Theorem \ref{th:ofzeroredundCMIM}.

To summarize, the approximations taken by all methods make them lose the good properties exhibited by the target objective functions, namely the assurance that features are correctly ordered, the existence of a stopping criterion, and the possibility to discard redundant features (here the exception is CMIM, in the weaker context of Theorem \ref{th:ofzeroredundCMIM}).

\subsection{Other methods}
\label{subsec:othermethods}

We now briefly discuss other methods that have appeared in the literature, explaining why they have not been included as part of the representative methods presented previously. 

MIFS-U \citep{citeulike:2607721} differs from MIFS in the estimation of the MI between the candidate feature and the class---this is a meaningless difference for the type of theoretical properties of the methods that we intend to address, in which estimation does not play a role. MIFS-ND \citep{hoque2014mifs} considers the same reference terms as mRMR, employing a genetic algorithm to select the features, thus again not changing anything in theoretical terms. ICAP \citep{jakulin2005machine} is similar to CIFE, while forcing the terms ${\rm TMI}(C,X_i,X_s)$, $X_s\in \boldsymbol{S}$, to be seen as redundancy terms by only considering their contribution when they are positive (negative for the objective function). IGFS \citep{el2008powerful} chooses the same candidate features in each step as JMI; and CMIM-2 \citep{vergara2010cmim} is also just the same as JMI, as its objective function is defined exactly as \eqref{eq:OF_JMI}.

A particular type of methods that were also not considered as representative is characterized by considering similar objective functions to those of the introduced representative methods, with the difference that all MI terms are replaced by corresponding normalized MI terms. More concretely: NMIFS \citep{estevez2009normalized} is an enhanced version of MIFS, MIFS-U, and mRMR; DISR \citep{meyer2006use} is adapted from JMI, and considers a type of normalization called symmetrical relevance; NJMIM \citep{bennasar2015feature} is adapted from JMIM, using also symmetrical relevance. Past experiments \citep[cf.][]{Brown:2012:CLM:2188385.2188387,bennasar2015feature} show that such normalizations make the methods more expensive, due to associated extra computations, with no compensation in terms of performance. In fact, it is argued in the mentioned experiments that the performance of such methods is actually worse than the performance of the corresponding methods that do not use normalized MI, which should be, as added by \cite{Brown:2012:CLM:2188385.2188387}, related to the additional variance introduced by the estimation of the extra normalization term.

%% file: section6_ArXiv.tex
\section{Comparison of feature selection methods on a distributional setting}
\label{sec:setting}

This section compares the feature selection methods using a distributional setting, based on a specific definition of class, features, and a performance metric. The setting provides an ordering for each of the methods, which is independent of specific datasets and estimation methods, and is compared with the ideal feature ordering. The aim of the setting is to illustrate how the drawbacks of the methods lead to incorrect feature ordering and to the loss of the good properties of the target objective functions.

We start by introducing a performance measure for feature selection methods that does not rely on the specificities of a fixed classifier---the \textit{minimum Bayes risk} (MBR). We then describe the characteristics of the setting, namely the definitions of class and features, and show how the quantities required to calculate the objective functions of the methods, i.e. the various types of MI, are calculated. Finally, we present and discuss the results.

\subsection{Minimum Bayes risk}
\label{subsec:otpm}

Commonly, the performance measures used to compare forward selection methods depend on how a particular classifier performs for certain data sets. As a result, it is not clear if the obtained conclusions are exclusively explained by the characteristics of the feature selection method, or if the specificities of the classifier and/or the data under study create confounding effects. To overcome this limitation, we consider a different type of performance measure that is computed at each step of the forward selection method under consideration. Using the set of selected features until a given step, we obtain, for a fixed classifier, the associated \textit{Bayes risk} (BR) or \textit{total probability of misclassification} \citep[Ch. 11]{johnson2002applied}. Bayes risk is a theoretical measure in the sense that it does not rely on data but instead directly on the, assumed to be known, distributions of the involved features into consideration \citep[see][for practical contexts where it was used]{kumar2004minimum,goel2000minimum}. The \textit{Bayes classifier} \citep[see][Ch. 1]{MR2422423} is a classifier that defines a classification rule associated with the minimum Bayes risk, which will be our performance measure. The suitability of this measure to our setting results from the fact that it relies on the distributions of the features, and also on their class-conditional distributions.

\subsubsection{Bayes risk and Bayes classifier}

For a given class $C$, with values on the set $\{0,1,...,c\}$, and a set of selected features $\boldsymbol{S}$, with support $\mathcal{S}$, a $(C,\boldsymbol{S})$-classifier $g$ is a (Borel-)measurable function from 
$\mathcal{S}$ to $\{0,1,...,c\}$, and $g(\boldsymbol{s})$ denotes the value of $C$ to which the observation $\boldsymbol{s}$ is assigned by the classifier. Then, the Bayes risk of the $(C,\boldsymbol{S})$-classifier $g$ is given by
\begin{equation*}
{\rm BR}(C,\boldsymbol{S},g)=P(g(\boldsymbol{S})\neq C))=\sum_{j=0}^{c}P(g(\boldsymbol{S})\neq j|C=j)P(C=j).
\end{equation*}

The proposed performance evaluation measure consists of the minimum possible value of the BR, called \textit{minimum Bayes risk} (MBR). Thus, for a given class $C$ and a set of selected features $\boldsymbol{S}$, the associated minimum Bayes risk, ${\rm MBR}(C,\boldsymbol{S})$, is given by: 
\begin{equation*}
{\rm MBR}(C,\boldsymbol{S})=\min_{g}{\rm BR}(C,\boldsymbol{S},g).
\label{eq:otpm}
\end{equation*}

The minimum Bayes risk corresponds to the Bayes risk of the so-called \textit{Bayes classifier}. The $(C,\boldsymbol{S})$ Bayes classifier assigns an object $\boldsymbol{s}\in \mathcal{S}$ to the value that $C$ is most likely to take given that $\boldsymbol{S}=\boldsymbol{s}$. That is, the $(C,\boldsymbol{S})$ Bayes classifier
$g$ is such that
\begin{align*}
g(\boldsymbol{s})&=\argmax_{j\in \{0,1,...,c\}} P(C=j|\boldsymbol{S}=\boldsymbol{s})\\
&=\argmax_{j\in \{0,1,...,c\}} P(C=j) f_{\boldsymbol{S}|C=j}(\boldsymbol{s}).
\end{align*}
Note that, in particular, when there are two possible values for the class, i.e. $c=1$, the Bayes classifier $g$ is such that \citep[see][Ch. 11]{johnson2002applied}:
\begin{equation}
g(\boldsymbol{s})=1\iff\frac{f_{\boldsymbol{S}|C=0}(\boldsymbol{s})}{f_{\boldsymbol{S}|C=1}(\boldsymbol{s})}\leq \frac{P(C=1)}{P(C=0)}. 
\label{eq:otpmequation}
\end{equation}

\subsubsection{Properties of the minimum Bayes risk}

We now discuss a few properties of the minimum Bayes risk, the proposed performance evaluation criterion. In the following, measurable should be read as Borel-measurable.

\begin{theorem}
If $C$ is a measurable function of $\boldsymbol{S}$, then ${\rm MBR}(C,\boldsymbol{S})=0$.
\label{th:tpmallinfo}
\end{theorem}
\begin{proof}
Let $g$ be the measurable function such that $C=g(\boldsymbol{S})$. As $C=g(\boldsymbol{S})$, it follows that ${\rm BR}(C,\boldsymbol{S},g)=P(g(\boldsymbol{S})\neq C))=0$. As $\rm{MBR}(C,\boldsymbol{S})$ is non-negative and $\rm{MBR}(C,\boldsymbol{S})\leq \rm{BR}(C,\boldsymbol{S},g)$, we conclude that ${\rm MBR}(C,\boldsymbol{S})=0$, as intended.
\end{proof}

Note that $C$ being a measurable function of $\boldsymbol{S}$ is equivalent to saying that features in $\boldsymbol{S}$ fully explain the class.

\begin{theorem}
If $X_i$ is a measurable function of $\boldsymbol{S}$, then ${\rm MBR}(C,\boldsymbol{S} \cup \{X_i\})={\rm MBR}(C,\boldsymbol{S})$.
\label{th:tpmred}
\end{theorem}
\begin{proof}
Let $\xi$ be the measurable function such that $X_i=\xi(\boldsymbol{S})$, and $g$ be the $(C,\boldsymbol{S}\cup\{X_i\})$ Bayes classifier, so that, in particular, ${\rm MBR}(C,\boldsymbol{S}\cup\{X_i\})={\rm BR}(C,\boldsymbol{S}\cup\{X_i\},g)$. Let $g'$ be the $(C,\boldsymbol{S})$ classifier such that, given an observation $\boldsymbol{s}\in \mathcal{S}$, $g'(\boldsymbol{s})=j$ when $g(\boldsymbol{s},\xi(\boldsymbol{s}))=j$. Then ${\rm BR}(C,\boldsymbol{S} \cup \{X_i\},g)={\rm BR}(C,\boldsymbol{S},g')$. As a consequence, by transitivity, ${\rm MBR}(C,\boldsymbol{S}\cup\{X_i\})={\rm BR}(C,\boldsymbol{S},g')$. This implies that ${\rm MBR}(C,\boldsymbol{S})\leq {\rm MBR}(C,\boldsymbol{S} \cup \{X_i\})$. In turn, it always holds that ${\rm MBR}(C,\boldsymbol{S})\geq{\rm MBR}(C,\boldsymbol{S} \cup \{X_i\})$ since $\boldsymbol{S}\subset \boldsymbol{S} \cup \{X_i\}$. Therefore, ${\rm MBR}(C,\boldsymbol{S} \cup \{X_i\})={\rm MBR}(C,\boldsymbol{S})$.
\end{proof}

Note that $X_i$ being a measurable function of $\boldsymbol{S}$ is equivalent to saying that $X_i$ is redundant given $\boldsymbol{S}$.

\subsection{Setting description}
\label{subsec:setting}

We now describe the distributional setting used to illustrate the various deficiencies of the feature selection methods. The class chosen for this setting is a generalization of the one proposed by \cite{claudiapaper}, which was based on the scenario introduced by \cite{citeulike:2607721} and later used by \cite{MR2422423}. It is defined as 
\begin{equation}\label{eq:class}
C_k = \left\{ \begin{array}{rl}
 0, &\mbox{$X+kY<0$} \\
 1, &\mbox{$X+kY\geq0$}
       \end{array}, \right.
\end{equation}
where $X$ and $Y$ are independent features with standard normal distributions and $k \in \, ]0,+\infty[$.

According to the discussion in Section \ref{sec:problem}, our scenario includes fully relevant, relevant, redundant, and irrelevant features. Specifically, our features are $X$, $X-k'Y$, $k'>0$, $Z$ and $X_{{\rm disc}}$. $X$ and $X-k'Y$ were chosen as relevant features that, taken together, fully explain the class. As irrelevant feature, we chose $Z$, independent of $X$ and $Y$, which for simplicity is considered to follow a Bernoulli distribution with success probability $1/2$. Finally, as redundant feature we chose
\begin{equation}
\label{eq:xdisc}
X_{{\rm disc}} = \left\{ \begin{array}{rl}
0, &\mbox{$X<0$} \\
1, &\mbox{$X\geq0$}
\end{array}, \right.
\end{equation}

The first selected feature is the candidate $X_i$ that has the largest value of ${\rm MI}(C_k,X_i)$. The possible candidates are $X$, $X-k'Y$, and $X_{{\rm disc}}$, which are the initially relevant features. $Z$ is an irrelevant feature and, therefore, will not be selected first. We want $X$ to be selected first to assure that, at the second step of the algorithm, there will be, as candidates, one fully relevant, one redundant, and one irrelevant feature. This provides an extreme scenario, where the relevancy level of the relevant feature is the maximum possible, making a wrong decision the hardest to occur. We next discuss the conditions for selecting $X$ before $X-k'Y$ and before $X_{{\rm disc}}$.

$X$ is selected before $X-k'Y$ if ${\rm MI}(C_k,X)>{\rm MI}(C_k,X-k'Y)$, which is equivalent to the condition
\begin{equation*}
	\arctan{k}< (\pi-\arctan{k'})-\arctan{k},
\end{equation*}
where the left term represents the angle between the lines $X=0$ and $X+kY=0$ and the right term represents the angle between the lines $X-k'Y=0$ and $X+kY=0$, in the context of the two-dimensional space defined by the pair of orthogonal vectors $(X,Y)$. The condition can be written in terms of $k$ as
\begin{equation}
k<\tan\left(\frac{\pi-\arctan{k'}}{2}\right).
\label{eq:xfirst}
\end{equation}

Feature $X_{{\rm disc}}$ is never selected before $X$ since ${\rm MI}(C_k,X_{{\rm disc}})\leq {\rm MI}(C_k,X)$ for all $k>0$. To see this, note that this inequality can be written, using \eqref{eq:micondproperty}, as $H(C_k)-H(C_k|X_{{\rm disc}})\leq H(C_k)-H(C_k|X)$, which is equivalent to $H(C_k|X_{{\rm disc}})\geq H(C_k|X)$. This is equivalent to $H(C_k|X_{{\rm disc}})\geq H(C_k|X_{{\rm disc}},X)$, which holds by \eqref{eq:infodonthurtgen}. In turn, $H(C_k|X) = H(C_k|X_{{\rm disc}},X)$ is equivalent to ${\rm MI}(C_k,X_{{\rm disc}}|X)=0$ by \eqref{eq:miconddiscrete}. Finally, since $X_{{\rm disc}}$ is redundant given $\{X\}$, equations \eqref{eq:miconddiscrete} and \eqref{eq:infodonthurtgen} can be used to verify that ${\rm MI}(C_k,X_{{\rm disc}}|X)=0$.

In view of the above discussion, the ideal feature ordering coming out of the distributional setting is $X$ in first place and $X-k'Y$ in second place. Ideally, the feature selection method should stop at this step, since a fully relevant feature has been found. However, further steps need to be considered, since actual methods do not preserve the stopping criterion, as discussed in Section \ref{subsec:propsmethods}. Then, at the third step, the remaining features, $Z$ and $X_{{\rm disc}}$, must be equally likely to be selected since, according to Theorems \ref{th:propsstandard}.2 and \ref{th:propsstandard}.3, both their target objective functions reach the lower bound.

\subsection{Required quantities}
\label{subsec:requiredquantities}

In order to be able to determine the order in which the features are selected by the different methods, we have to derive expressions, depending on $k$ and $k'$, needed for evaluating the corresponding objective functions. We need: the MI between each candidate feature and the class, the MI between different pairs of candidate features, and the class-conditional MI between pairs of candidate features. The computation of these quantities require obtaining the univariate entropies of the candidate features and of the class. The derivations of such expressions are provided in Appendix  \ref{app:A1} 
 and their final forms are available in Tables \ref{tab:entropyvariables} to \ref{tab:miinputinputcond}.

\begin{paragraph}{Univariate entropies}

We start with a summary of the univariate entropies of the different features and of the class presented in Table \ref{tab:entropyvariables}. The corresponding derivations can be found in Appendix \ref{app:A1}.

\begin{table}[t!]
\centering
\caption{Entropies of the class, $C_k$, and the input features.}
\footnotesize
\begin{tabular}{l ccccc}
\hline
 & $C_k$ & $X$ & $X-k'Y$ & $Z$ & $X_{{\rm disc}}$\\
\midrule
Entropy & $\ln(2)$ & $\frac{1}{2}\ln(2\pi e)$ & $\frac{1}{2}\ln(2\pi e (1+k^{'2}))$ & $\ln(2)$ & $\ln(2)$\\
\hline
\end{tabular}
\label{tab:entropyvariables}
\end{table}

\end{paragraph}

\begin{paragraph}{MI between input features and the class}

As for the MI between input features and the class, they are provided in Table \ref{tab:miinputclass}. The corresponding derivations can be found in Appendix \ref{app:A2}.

\begin{table}[t!]
\centering
\caption{MI between each input feature and the class, $C_k$.}
\footnotesize
\label{tab:miinputclass}
\begin{threeparttable}
\begin{tabular}{c c}
\hline
$A$ & ${\rm MI}(C_k,A)$\\
\midrule
$X$ & $\frac{1}{2}\ln(2\pi e)-\frac{1}{2}\sum_{j=0}^1 \int_{\mathbb{R}} f_{X|C_k=j}(u) \ln f_{X|C_k=j}(u) du$ \\[0.1cm]
$X-k'Y$ & $\frac{1}{2}\ln(2\pi e (1+k^{'2}))-\frac{1}{2}\sum_{j=0}^1 \int_{\mathbb{R}} f_{X-k'Y|C_k=j}(u) \ln f_{X-k'Y|C_k=j}(u) du$  \tnote{a}\\[0.1cm]
$Z$ &  $0$ \\[0.1cm]
$X_{{\rm disc}}$ & $2\ln(2)+\frac{\arctan{k}}{\pi}\ln(\frac{\arctan{k}}{2\pi})+(1-\frac{\arctan{k}}{\pi})\ln(\frac{1}{2}-\frac{\arctan{k}}{2\pi})$ \tnote{b}\\[0.1cm]
\bottomrule
\end{tabular}
\begin{tablenotes}
\item [a] $X|C_k=j\sim {\rm SN}(0,1,\frac{(-1)^{j+1}}{k})$, $j=0,1$.
\item [b] $X-k'Y|C_k=j\sim {\rm SN}(0,\sqrt{1+k'^2},(-1)^{j+1}(\frac{1-kk'}{k+k'}))$, $j=0,1$.
\end{tablenotes}
\end{threeparttable}
\end{table}

It must be added that the notation $W\sim {\rm SN}(\mu,\sigma,\alpha)$ (where $\mu\in \mathbb{R}$, $\sigma>0$, and $\alpha\in \mathbb{R}$), means that the random variable $W$ follows a skew-normal distribution, so that it has probability density function \citep{MR877720} 
\begin{equation}
f_W(w)=\frac{2}{\sigma}\phi(\frac{w-\mu}{\sigma})\Phi(\frac{\alpha(w-\mu)}{\sigma}), \quad w\in \mathbb{R},
\label{eq:skew}
\end{equation}
where $\Phi(z)$ denotes the value of the standard normal distribution function at point $z$, while $\phi(z)$ denotes the probability density function, for the same distribution, also at $z$.

\end{paragraph}

\begin{paragraph}{MI between pairs of input features}

As for the MI between the different pairs of input features, they are provided in Table \ref{tab:miinputinput}. The corresponding derivations can be found in Appendix  \ref{app:A3}.

\begin{table}[t!]
\centering
\caption{MI between pairs of input features.}
\footnotesize
\label{tab:miinputinput}
\begin{threeparttable}
\begin{tabular}{cc c}
\hline
$A$ & $B$ & ${\rm MI}(\cdot,\cdot)$ \\[0.4ex]
\midrule
$X$ & $X-k'Y$ & $\frac{1}{2}\ln(1+\frac{1}{k'^2})$ \\[0.1cm]
$X$ & $X_{{\rm disc}}$ & $\ln(2)$ \\[0.1cm]
$X-k'Y$ & $X_{{\rm disc}}$ & $\frac{1}{2}\ln(2\pi e)-\frac{1}{2}\sum_{j=0}^1 \int_{\mathbb{R}} f_{X|C_{k'}=j}(u) \ln f_{X|C_{k'}=j}(u) du$ \tnote{a}\\[0.1cm]
$Z$ & $B$ & $0$, \hspace{1ex} $B\in \{X,X-k'Y,X_{{\rm disc}}\}$ \\[0.1cm]
\arrayrulecolor{black}\hline
\hline
\end{tabular}
\begin{tablenotes}
\item [a] $X|C_{k'}=j\sim {\rm SN}(0,1,\frac{(-1)^{j+1}}{k'})$, $j=0,1$.
\end{tablenotes}
\end{threeparttable}
\end{table}

\end{paragraph}

\begin{paragraph}{Class-conditional MI between pairs of input features}

As for the class-conditional MI between the different pairs of input features, they are provided in Table \ref{tab:miinputinputcond}. The corresponding derivations can be found in Appendix  \ref{app:A4}.

\begin{table}[t!]
\centering
\caption{Class-conditional MI between pairs of input features.}
\footnotesize
\label{tab:miinputinputcond}
\begin{threeparttable}
\begin{tabular}{cc c}
\hline
$A$ & $B$ & ${\rm MI}(\cdot,\cdot|C_k)$ \\[0.4ex]
\midrule
$X$ & $X-k'Y$ &\begin{tabular}{c}$\frac{1}{2}\sum_{j=0}^1 \int_{\mathbb{R}} f_{X|C_k=j}(u) \ln f_{X|C_k=j}(u) du+$\\[0.1cm]
                                 $\frac{1}{2}\sum_{j=0}^1 \int_{\mathbb{R}} f_{X-k'Y|C_k=j}(u) \ln f_{X-k'Y|C_k=j}(u) du-$\\[0.1cm]
                                 $(1+\ln{\pi}+\ln{k'})$\tnote{a,b}\\[0.1cm] \end{tabular}\\ 
\arrayrulecolor{black!30}\midrule
$X$ & $X_{{\rm disc}}$ & $-\frac{\arctan{k}}{\pi}\ln(\frac{\arctan{k}}{\pi})-(1-\frac{\arctan{k}}{\pi})\ln(1-\frac{\arctan{k}}{\pi})$ \\[0.1cm]
\midrule
$X-k'Y$ & $X_{{\rm disc}}$ & $\frac{1}{2}\sum_{j=0}^1 \int_{\mathbb{R}} f_{X-k'Y|C_k=j}(u) \ln f_{X-k'Y|C_k=j}(u) du-h(X-k'Y|X_{{\rm disc}},C_k)$ \tnote{b,c}\\[0.1cm]
\midrule
$Z$ & $B$ & $0$, \hspace{1ex} $B\in \{X,X-k'Y,X_{{\rm disc}}\}$ \\[0.1cm]
\arrayrulecolor{black}\hline
\end{tabular}
\begin{tablenotes}
\item [a] $X|C_{k'}=j\sim {\rm SN}(0,1,\frac{(-1)^{j+1}}{k'})$, $j=0,1$.
\item [b] $X-k'Y|C_k=j\sim {\rm SN}(0,\sqrt{1+k'^2},(-1)^{j+1}(\frac{1-kk'}{k+k'}))$, $j=0,1$.
\item [c] $h(X-k'Y|X_{{\rm disc}},C_k)$ in  \eqref{eq:giantentropy}.
\end{tablenotes}
\end{threeparttable}
\end{table}

\end{paragraph}

\subsection{Applying the different feature selection methods}
\label{subsec:ordervariables}

We now present the results of applying the various feature selection methods to the distributional setting. The feature ordering will be discussed for different values of $k$ and $k'$. Taking the objectives of this study into consideration, for fixed $k'$, the most interesting case is that where $X$ alone leaves the largest possible amount of information undetermined about the class; this leads to $X-k'Y$ having the most importance in the explanation of the class, making the error of not choosing it after $X$ the worst possible. According to the performance metric introduced in Section \ref{subsec:otpm}, we want to choose a value of $k$ that leads to a large MBR when $X$ is the only selected feature, ${\rm MBR}(C_k,\{X\})$, which is given by (see \ref{app:B}):
\begin{equation}
{\rm MBR}(C_k,\{X\})=\frac{\arctan{k}}{\pi}.
\label{eq:otpmx}
\end{equation}
Since ${\rm MBR}(C_k,\{X\})$ is an increasing function of $k$, we want $k$ to be as large as possible, under the restriction \eqref{eq:xfirst}. We consider $k=\tan\left((\pi-\arctan{k'}-10^{-6})/2\right)$.

Given that, in our setting, the features $X$ and $X-k'Y$ fully explain the class, and that, according to Theorem \ref{th:tpmallinfo}, ${\rm MBR}(C_k,\{X,X-k'Y\})=0$, it makes sense to take the MBR based on the first two selected features as performance measure for characterizing each forward feature selection method. This measure is denoted by ${\rm MBR}_2$.

We will carry out two different studies. In the first one, we concentrate on the methods that ignore complementarity; i.e. MIFS, mRMR, and maxMIFS, and study the feature ordering as a function of $k'$. The purpose of this study is to highlight the consequences of ignoring complementarity. In the second study, we compare the feature ordering of all methods under analysis, for fixed $k$ and $k'$. The goal is to provide examples showing wrong decisions made by the various methods.

\subsubsection{The consequences of ignoring complementarity}

In this section, we address the consequences of ignoring complementarity, as a function of $k'$; thus, we concentrate on methods MIFS ($\beta=1$), mRMR, and maxMIFS. The motivation for scanning $k'$ is that it provides different levels of association between the already selected feature $X$ and the candidate feature $X-k'Y$. For small values of $k'$, $X$ and $X-k'Y$ are strongly associated, and the level of association decreases as $k'$ increases.

Scanning $k'$ from $0$ to $+\infty$ defines three regions, each corresponding to a specific feature ordering. This is shown in Figure \ref{fig:diffparameters}, where $k'$ was scanned with a step size of $0.01$, starting at $0.01$. To complete the discussion, we include the objective function values in the second step of the algorithms in Figure \ref{fig:diffofs}, and the corresponding ${\rm MBR}_2$ values in Figure \ref{fig:oerk}. Note that, at this step, the objective function takes the same value for all candidate features and methods and ${\rm MBR}_2$ takes the same value for all methods.

\begin{figure}[t!]
\centering
\begin{tikzpicture}[domain=-6:6]
\filldraw[gray!25](3,0)--(3+2*0.575,2)--(3,2)--cycle;
\filldraw[gray!50](3,0)--(3+2*0.575,2)--(3+2,2)--(3+2,2/2.115)--cycle;
\filldraw[gray!75](3,0)--(5,0)--(5,2/2.115)--cycle;
\draw[smooth,thick,domain=3-2*0.575:3+2*0.575]plot(\x,{(1/0.575)*(\x-3)}) node[right] {$k'\approx 0.575$};
\draw[smooth,thick,domain=1:5]plot(\x,{(1/2.115)*(\x-3)}) node[right] {$k'\approx 2.115$};
\draw[->] (1,0) -- (5,0) node[right] {$x$};
 \draw[->] (3,-2) -- (3,2) node[above] {$y$};
 \node[] at (3+0.3,1.2){$\boldsymbol{(a)}$};
  \node[] at (3+1.4,1.4){$\boldsymbol{(b)}$};
   \node[] at (3+1.25,0.6/2.12){$\boldsymbol{(c)}$};
   
\filldraw[gray!25](-3,0)--(-3+2*0.575,2)--(-3,2)--cycle;
\filldraw[gray!50](-3,0)--(-3+2*0.575,2)--(-3+2,2)--(-3+2,2/2.115)--cycle;
\filldraw[gray!75](-3,0)--(-1,0)--(-1,2/2.115)--cycle;
\draw[smooth,thick,domain=-3-2*0.565:-3+2*0.565]plot(\x,{(1/0.565)*(\x+3)}) node[right] {$k'\approx 0.565$};
\draw[smooth,thick,domain=-5:-1]plot(\x,{(1/2.115)*(\x+3)}) node[right] {$k'\approx 2.115$};
\draw[->] (-5,0) -- (-1,0) node[right] {$x$};
 \draw[->] (-3,-2) -- (-3,2) node[above] {$y$};
 \node[] at (-3+0.3,1.2){$\boldsymbol{(a)}$};
  \node[] at (-3+1.4,1.4){$\boldsymbol{(b)}$};
   \node[] at (-3+1.25,0.6/2.12){$\boldsymbol{(c)}$};
  
\end{tikzpicture}
\caption{Regions associated with a specific ordering of the features, defined by the values of $k'$. In this representation, $k'$ is defined through the line $x-k'y=0$. For region $(a)$, the ordering is $\{X,Z,X_{{\rm disc}},X-k'Y\}$; for $(b)$, it is $\{X,Z,X-k'Y,X_{{\rm disc}}\}$; and for $(c)$, the ordering is already correct since $X$ is chosen first and $X-k'Y$ second. For methods MIFS and maxMIFS (resp. mRMR), represented in the left (resp. right), $(a)$ is associated with $0<k'<0.575$ (resp. $0<k'<0.565$) and $(b)$ with $0.575<k'<2.115$ (resp. $0.565<k'<2.115$). Region $(c)$ is associated with $k'>2.115$ for the three methods.}
\label{fig:diffparameters}
\end{figure}
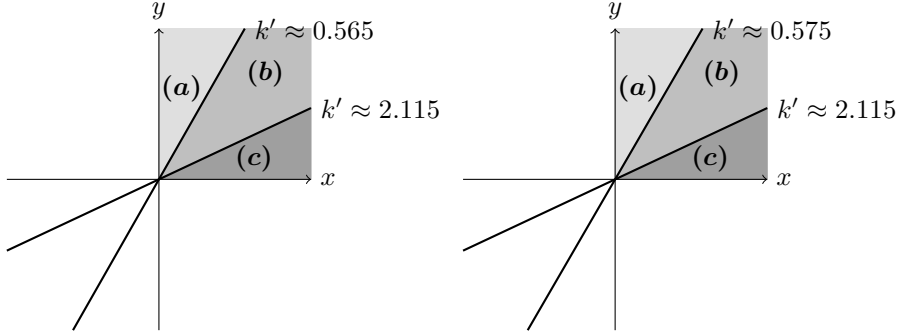

For small values of $k'$, smaller than $0.565$ for MIFS and maxMIFS, and than $0.575$ for mRMR---region (a)---the feature ordering is  $X$, $Z$, $X_{{\rm disc}}$, $X-k'Y$. In this region, $X-k'Y$ is chosen last due to a large inter-feature redundancy with $X$. As shown in Figure \ref{fig:diffofs}, in this region, the objective function values of $X-k'Y$ and $X_{{\rm disc}}$ at the second step are negative and smaller than that of $Z$, which explains why $Z$ is selected in second place. At the third step, the objective function values of $X-k'Y$ and $X_{{\rm disc}}$ are exactly the same as in the second step for MIFS and maxMIFS, and only slightly different for mRMR. Thus, in this region, the objective functions of $X-k'Y$ are more negative than those of $X_{{\rm disc}}$, which explains why $X_{{\rm disc}}$ is selected in third place.

For intermediate values of $k'$, smaller than $2.115$, and larger than $0.565$ for MIFS and maxMIFS and than $0.575$ for mRMR---region (b)---the feature ordering is $X$, $Z$, $X-k'Y$, $X_{{\rm disc}}$. In this region, the objective functions of $X-k'Y$ are larger than those of $X_{{\rm disc}}$, but smaller than those of $Z$.

For large values of $k'$, larger than $2.115$---region (c)---the correct feature ordering is achieved since $X-k'Y$ is selected in second place. Note that in this region, there are two possible orderings for $Z$ and $X_{{\rm disc}}$, but this issue is not relevant for our discussion.

The problem of these methods in regions (a) and (b) is due to the lack of the class-relevant redundancy term in their objective functions, which expresses the complementarity effects. In fact, the association between $X$ and $X-k'Y$, as measured by ${\rm MI}(X-k'Y,X)$, grows significantly as $k'$ approaches $0$, but so does the class-relevant redundancy, which is given by ${\rm MI}(X-k'Y,X|C_k)$. Ignoring the compensation given by the latter term leads to objective function values that can take negative values. Moreover, this explains why the lower bound of $0$ from Theorem \ref{th:boundsgeneral}, associated with the target objective function, is lost for these methods. Also, in contradiction with the good properties of the target objective function, the objective functions of these methods do not take the same (minimum) values at $X_{{\rm disc}}$ and $Z$. The ${\rm MBR}_2$ values of these methods (see Figure \ref{fig:oerk}) confirm that the performance is very poor in regions (a) and (b): it is above $0.4$ in region (a) and above $0.3$ in region (b). These results show that ignoring complementarity is a severe drawback that can lead to gross errors in the feature selection process.

\begin{figure}[t!]
\centering
\includegraphics[scale=0.5]{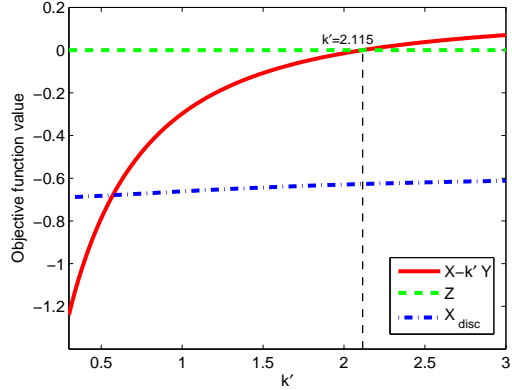}
\caption{Evaluation of the objective function for the different candidate features in the second step of the algorithms (MIFS, mRMR, and maxMIFS) depending on the value of $k'$.}
\label{fig:diffofs}
\end{figure}

Figure \ref{fig:diffofs} also shows that results analogous of Theorem \ref{th:ofzeroredundCMIM} do not hold for MIFS, mRMR, and maxMIFS. In fact, the objective function at $X_{{\rm disc}}$, a redundant feature, is not necessarily the minimum; in particular, this happens for small values of $k'$, where the objective function at $X-k'Y$ takes lower values than that at $X_{{\rm disc}}$.

\begin{figure}[t!]
\centering
\includegraphics[scale=0.5]{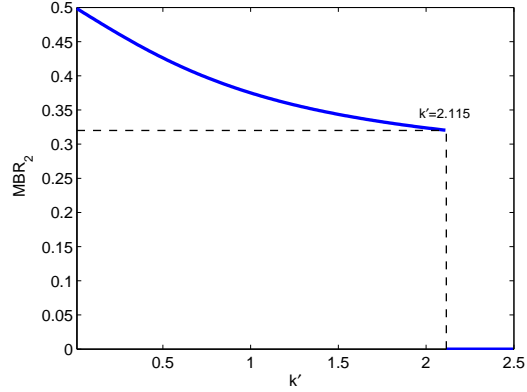}
\caption{${\rm MBR}_2$ for the different algorithms (MIFS, mRMR, and maxMIFS) depending on the value of $k'$. For $k'<2.115$, ${\rm MBR}_2=\frac{\pi-\arctan{k'}-10^{-6}}{2\pi}$ using $k=\tan\left((\pi-\arctan{k'}-10^{-6})/2\right)$ and \eqref{eq:otpmx}; for $k'>2.115$, the right second feature is chosen, $X-k'Y$, so that ${\rm MBR}_2=0$.}
\label{fig:oerk}
\end{figure}

\subsubsection{Feature ordering}

We now compare the feature ordering of all methods, for fixed $k$ and $k'$. In order to place the methods in a challenging situation, we use $k$ and $k'$ values that maximize the MBR. Per \eqref{eq:otpmx} we need to maximize $k$ and per \eqref{eq:xfirst} we need to minimize $k'$. We choose for $k'$ the first value of the grid used in the context of Figure \ref{fig:diffparameters}, i.e. $k'=0.01$. In this case, $k=\tan(\frac{\pi - \arctan{k'}-10^{-6}}{2})=199.985$ and ${\rm MBR}(C_k,\{X\})\approx 0.498$. Recall that, since \eqref{eq:xfirst} holds, and therefore ${\rm MI}(C_k,X)>{\rm MI}(C_k,X-k'Y)$, $X$ is always selected in first place. 

Table \ref{tab:varsel} shows the feature ordering and the associated values of ${\rm MBR}_2$. The ordering of features relates to the concrete values of the terms that compose the objective functions. These are provided in Table \ref{tab:miclass}, which contains the values of MI between each candidate feature and the class; Table \ref{tab:miinput}, which contains the values of MI between the different features; and Table \ref{tab:micond}, which contains the values of the class-conditional MI between the different input features. Note that, in Table \ref{tab:miclass}, ${\rm MI}(C_k,X)$, ${\rm MI}(C_k,X-k'Y)$, and ${\rm MI}(C_k,X_{{\rm disc}})$ are all shown as taking approximately the value $0$, but actually ${\rm MI}(C_k,X)$ is the largest one.

\begin{table}[t!]
\centering
\caption{Feature ordering and corresponding ${\rm MBR}_2$, for $k=199.985$ and $k'=0.01$.}
\footnotesize
\begin{tabular}{c ccccc}
\hline
Methods & \multicolumn{4}{c}{Order of feature selection} & ${\rm MBR}_2$\\
\midrule
MIM &  $X$ & $X-k'Y$ & $X_{{\rm disc}}$ & $Z$ & $0$\\
MIFS ($\beta=1$) & $X$ & $Z$ & $X_{{\rm disc}}$ & $X-k'Y$ & $0.498$\\[0.1cm]
mRMR   & $X$ & $Z$ & $X_{{\rm disc}}$ & $X-k'Y$ & $0.498$\\[0.1cm]
maxMIFS   & $X$ & $Z$ & $X_{{\rm disc}}$ & $X-k'Y$ & $0.498$\\[0.1cm]
CIFE & $X$ & $X-k'Y$ & $X_{{\rm disc}}$ & $Z$ & $0$\\[0.1cm]
JMI &  $X$ & $X-k'Y$ & $X_{{\rm disc}}$ & $Z$ & $0$\\[0.1cm]
CMIM & $X$ & $X-k'Y$ & $Z$/$X_{{\rm disc}}$ & $X_{{\rm disc}}$/$Z$ & $0$\\[0.1cm]
JMIM & $X$ & $X-k'Y$ & $X_{{\rm disc}}$ & $Z$ & $0$\\[0.1cm]
\bottomrule
\end{tabular}
\label{tab:varsel}
\end{table}

\begin{table}[t!]
\centering
\caption{MI between the class and each input feature, for $k=199.985$ and $k'=0.01$.}
\footnotesize
\begin{tabular}{c cccc}
\hline
 & $X$ & $X-k'Y$ & $Z$ & $X_{{\rm disc}}$\\
\midrule
${\rm MI}(\cdot,C_k)$ & $\approx 0$ & $\approx 0$ & $0$ & $\approx 0$\\
\bottomrule
\end{tabular}
\label{tab:miclass}
\end{table}

\begin{table}[t!]
\centering
\caption{MI between pairs of input features, for $k=199.985$ and $k'=0.01$.}
\footnotesize
\begin{tabular}{c ccc}
\hline
${\rm MI}(\cdot,\cdot)$ & $X$ & $X-k'Y$ & $Z$ \\
\midrule
$X-k'Y$    & $4.605$ & & \\
$Z$    & $0$ & $0$ & \\
$X_{{\rm disc}}$    & $0.693$ & $0.686$ & $0$ \\
\bottomrule
\end{tabular}
\label{tab:miinput}
\end{table}

\begin{table}[t!]
\centering
\caption{Class-conditional MI between pairs of input features, for $k=199.985$ and $k'=0.01$.}
\footnotesize
\begin{tabular}{c ccc}
\hline
${\rm MI}(\cdot,\cdot|C_k)$ & $X$ & $X-k'Y$ & $Z$ \\
\midrule
$X-k'Y$    & $5.298$ & & \\
$Z$    & $0$ & $0$ & \\
$X_{{\rm disc}}$    & $0.693$ & $0.689$ & $0$ \\
\bottomrule
\end{tabular}
\label{tab:micond}
\end{table}

Table \ref{tab:varsel} shows that all methods, except MIFS, mRMR, and maxMIFS, achieve an ${\rm MBR}_2$ of 0. However, the third step of the algorithm is only completely correct for CMIM. In fact, it should be equally likely to choose $Z$ or $X_{{\rm disc}}$, but CIFE, JMI, JMIM, and MIM select $X_{{\rm disc}}$ first.
MIM suffers from redundancy ignored drawback. The fact that the selection is correct at the first two steps of the feature selection process is meaningless; it only happens because ${\rm MI}(C_k,X-k'Y)$ is slightly larger than ${\rm MI}(C_k,X_{{\rm disc}})$.

The methods that ignore complementarity, i.e. MIFS, mRMR, and maxMIFS, fail at the second step of the feature selection process, by not selecting $X-k'Y$. For all methods, the objective function is $0$ for $Z$, ${\rm MI}(C_k,X-k'Y)-{\rm MI}(X-k'Y,X)=-4.605$ for $X-k'Y$, and ${\rm MI}(C_k,X_{{\rm disc}})-{\rm MI}(X_{{\rm disc}},X)=-0.693$ for $X_{{\rm disc}}$, which explains why $Z$ is selected at this step. Adding the class-relevant redundancy term to the objective functions, would make them take the value $\ln(2)$ for $X-k'Y$ and $0$ for $X_{{\rm disc}}$, leading to the selection of $X-k'Y$. In fact, the class-relevant redundancy term is ${\rm MI}(X-k'Y,X|C_k)=5.298$ for $X-k'Y$, and ${\rm MI}(X_{{\rm disc}},X|C_k)=0.693$ for $X_{{\rm disc}}$. Note that $\ln(2)$ is precisely the maximum of the target objective function, which is achieved for fully relevant features (the case of $X-k'Y$), and the minimum is $0$, achieved by irrelevant and redundant features (the cases of $Z$ and $X_{{\rm disc}}$). Thus, accounting for the class-relevant redundancy compensates the potentially large negative values associated with the inter-feature redundancy.

With the exception of CMIM, the methods that do not ignore complementarity, fail at the third step of the feature selection process by preferring $X_{{\rm disc}}$ over $Z$, as shown in Table \ref{tab:varsel}. 
 
As discussed in Section \ref{sec:methods}, CIFE suffers from overscaled redundancy drawback. At the third step of the feature selection process, after selecting $X$ and $X-k'Y$, the objective function for candidate feature $X_i$ is
\begin{equation}
{\rm MI}(C_k,X_i) - {\rm MI}(X_i,X) + {\rm MI}(X_i,X|C_k) - {\rm MI}(X_i,X-k'Y) + {\rm MI}(X_i,X-k'Y|C_k),
\label{eq:cifeobj}
\end{equation}
while the associated target objective function is
\begin{equation}
{\rm MI}(C_k,X_i) - {\rm MI}(X_i,\{X,X-k'Y\}) + {\rm MI}(X_i,\{X,X-k'Y\}|C_k).
\label{eq:target}
\end{equation}

Both objective functions take the value $0$ for the candidate feature $Z$. For the candidate $X_{{\rm disc}}$, the target objective function \eqref{eq:target} can be written as
\begin{equation}
{\rm MI}(C_k,X_i) - {\rm MI}(X_i,X) + {\rm MI}(X_i,X|C_k),
\label{eq:targetreduced}
\end{equation}
given that ${\rm MI}(X_{{\rm disc}},\{X,X-k'Y\})={\rm MI}(X_{{\rm disc}},X)$ and ${\rm MI}(X_{{\rm disc}},\{X,X-k'Y\}|C_k)={\rm MI}(X_{{\rm disc}},X|C_k)$. Concerning the first condition, note that ${\rm MI}(X_i,\boldsymbol{S})={\rm MI}(X_{{\rm disc}},\{X,X-k'Y\})=H(X_{{\rm disc}})-H(X_{{\rm disc}}|X,X-k'Y)=H(X_{{\rm disc}})-H(X_{{\rm disc}}|X)={\rm MI}(X_{{\rm disc}},X)$, since $H(X_{{\rm disc}}|X)=0$ implies that $H(X_{{\rm disc}}|X,X-k'Y)=0$ also, by \eqref{eq:infodonthurtgen}; a similar reasoning can be used to show that the second condition also holds.

Thus, in the case of $X_{{\rm disc}}$, we see that, when comparing the objective function of CIFE, given by \eqref{eq:cifeobj}, with the target objective function, given by \eqref{eq:targetreduced}, CIFE includes an extra part with two terms, $-{\rm MI}(X_i,X-k'Y) + {\rm MI}(X_i,X-k'Y|C_k)$, which is responsible for increasing the redundancy scale. In our case, the extra term takes the value $- {\rm MI}(X_{{\rm disc}},X-k'Y) + {\rm MI}(X_{{\rm disc}},X-k'Y|C_k)=-0.686+0.689=0.003$; recall Tables \ref{tab:miinput} and \ref{tab:micond}. This is exactly the value of the objective function at $X_{{\rm disc}}$, since the remaining terms sum to $0$, which explains why $X_{{\rm disc}}$ is selected before $Z$. To see that the remaining terms sum to $0$, note that these terms correspond exactly to the evaluation of the target objective function OF', and recall that the target objective function value must be $0$ for a redundant feature.

The amount of overscaling is relatively modest in this example but, clearly, the problem gets worse as $\boldsymbol{S}$ increases, since more terms are added to the objective function. We also note that, while in this case the objective function has been overvalued, it could have equally been undervalued. This fact together with the overscaling problem is what makes the objective function of CIFE not bounded, neither from below nor from above.

JMI tried to overcome the problem of CIFE by introducing the scaling factor $1/|\boldsymbol{S}|$ in the TMI approximation. However, as discussed in Section \ref{sec:methods}, this leads to redundancy undervalued drawback. At the third step of the feature selection process, the objective function of JMI is
\begin{equation*}
{\rm MI}(C_k,X_i) - \frac{1}{2} {\rm MI}(X_i,X) + \frac{1}{2} {\rm MI}(X_i,X|C_k) - \frac{1}{2} {\rm MI}(X_i,X-k'Y) + \frac{1}{2} {\rm MI}(X_i,X-k'Y|C_k)
\end{equation*}
for the candidate $X_i$. Its value equals $0$ for the candidate $Z$, but for candidate $X_{{\rm disc}}$ it equals $0-0.5\times 0.693+0.5\times 0.693-0.5\times 0.689+0.5\times 0.686=0.0015$, which explains why $X_{{\rm disc}}$ is selected before $Z$.

This results directly from the undervaluing of the terms ${\rm MI}(X_i,\boldsymbol{S})$ and ${\rm MI}(X_i,\boldsymbol{S}|C)$ of the target objective function at $X_i=X_{{\rm disc}}$. In fact, ${\rm MI}(X_i,\boldsymbol{S})={\rm MI}(X_{{\rm disc}},X)=0.693$, but JMI approximates it by a smaller value, i.e. $\frac{1}{2}{\rm MI}(X_{{\rm disc}},X)+\frac{1}{2}{\rm MI}(X_{{\rm disc}},X-k'Y)=\frac{1}{2}\times 0.693 + \frac{1}{2}\times 0.686 = 0.6895$. Similarly, ${\rm MI}(X_{{\rm disc}},\boldsymbol{S}|C_k)={\rm MI}(X_{{\rm disc}},X|C_k) = 0.693$, but again JMI approximates it by a smaller value, i.e. $\frac{1}{2}{\rm MI}(X_{{\rm disc}},X|C_k)+\frac{1}{2}{\rm MI}(X_{{\rm disc}},X-k'Y|C_k) = \frac{1}{2}\times 0.693 + \frac{1}{2}\times 0.689 = 0.691$.

JMIM introduced an additional term in the objective function which, as discussed in Section \ref{sec:methods}, is unimportant and may lead to confusion in the selection process---unimportant term approximated drawback. At the third step of the selection process, the objective function of JMIM is
\begin{align*}
{\rm MI}(C_k,X_i)-&\max\left\{{\rm MI}(X_i,X)- {\rm MI}(X_i,X|C_k)-{\rm MI}(C_k,X),\right.\\
&\left.\hspace*{6ex}{\rm MI}(X_i,X-k'Y)- {\rm MI}(X_i,X-k'Y|C_k)-{\rm MI}(C_k,X-k'Y)\right\},
\end{align*}
for candidate feature $X_i$. In this case, the objective function for candidate feature $Z$ is $$0-\max \left\{0-0-{\rm MI}(C_k,X), 0-0-{\rm MI}(C_k,X-k'Y)\right\},$$ and for candidate $X_{{\rm disc}}$ it is $0-\max \left\{0.693-0.693-{\rm MI}(C_k,X), 0.686-0.689-{\rm MI}(C_k,X-k'Y)\right\}$. We first note that ${\rm MI}(C_k,X)$ and ${\rm MI}(C_k,X-k'Y)$ are both approximately $0$, while ${\rm MI}(C_k,\{X,X-k'Y\})$, the quantity they try to approximate, takes the value $\ln(2)$. Since, per design of our experiment ${\rm MI}(C_k,X) > {\rm MI}(C_k,X-k'Y)$, it turns out that the objective function of $Z$ equals ${\rm MI}(C_k,X-k'Y)$, and that of $X_{{\rm disc}}$ equals ${\rm MI}(C_k,X)$, leading to the selection of $X_{{\rm disc}}$. There are two observations that should be pointed out. First, contrarily to the previous cases of CIFE and JMI, the objective function for $Z$ takes a value that is no longer according to the corresponding target objective function, which in this case should be ${\rm MI}(C_k,\{X,X-k'Y\})=\ln(2)$. Second, the choice between the two features, $Z$ and $X_{{\rm disc}}$, is being done by two terms, ${\rm MI}(C_k,X)$ and ${\rm MI}(C_k,X-k'Y)$, that try to approximate a term that does not depend on the candidate features, ${\rm MI}(C,\boldsymbol{S})$, and therefore should take the same value for both features and not become a deciding factor. 
The results regarding MIM, CIFE, JMI, and JMIM provide counter-examples showing that theorems analogous to Theorem \ref{th:ofzeroredundCMIM} do not hold for these methods. Indeed, in all cases, the objective function at $X_{{\rm disc}}$, a redundant feature, takes values different from the minimum of the corresponding objective function.

CMIM is the only method that performs correctly in the distributional setting. At the third step of the feature selection process, the objective function is $0$ for both $Z$ and $X_{{\rm disc}}$. The latter result can be obtained from Theorem \ref{th:ofzeroredundCMIM}, since $X_{{\rm disc}}$ is redundant given $\{X\}$. This can be confirmed numerically. The objective function of CMIM at the third step of the feature selection process is
\begin{equation*}
{\rm MI}(C_k,X_i)-\max\left\{{\rm MI}(X_i,X)- {\rm MI}(X_i,X|C_k),{\rm MI}(X_i,X-k'Y)- {\rm MI}(X_i,X-k'Y|C_k)\right\},
\end{equation*}
for candidate feature $X_i$. In this case, the objective function for candidate feature $Z$ is $0$, and the same holds for $X_{{\rm disc}}$ since $0-\max \left\{0.693-0.693, 0.686-0.689\right\}=0$. Note however that this does not mean that CMIM always performs correctly. As discussed in Section \ref{subsec:representative}, CMIM suffers from the problems of redundancy undervalued and complementarity penalized, and Example \ref{ex:problemCMIM} provides a case where CMIM decides incorrectly.

%% file: concs.tex
\section{Conclusions}
\label{sec:concs}

This paper carried out an evaluation and a comparison of forward feature selection methods based on mutual information. For this evaluation we selected methods representative of all types of feature selection methods proposed in the literature, namely MIM, MIFS, mRMR, maxMIFS, CIFE, JMI, CMIM, and JMIM. The evaluation was carried out theoretically, i.e. independently of the specificities of datasets and classifiers; thus, our results establish unequivocally the relative merits of the methods.

Forward feature selection methods iterate step-by-step and select one feature at each step, among the set of candidate features, the one that maximizes an objective function expressing the contribution each candidate feature to the explanation of the class. In our case, the mutual information (MI) is used as the measure of association between the class and the features. Specifically, the candidate feature selected at each step is the one that maximizes the MI between the class and the set formed by the candidate feature and the already selected features.

Our theoretical evaluation is grounded on a target objective function that the methods try to approximate and on a categorization features according to their contribution to the explanation of the class. The features are categorized as irrelevant, redundant, relevant, and fully relevant. This categorization has two novelties regarding previous works: first, we introduce the important category of fully relevant features; second, we separate non-relevant features in two categories of irrelevant and redundant features. Fully relevant features are features that fully explain the class and, therefore, its detection can be used as a stopping criterion of the feature selection process. Irrelevant and redundant features have different properties, which explains why we considered them separately. In particular, we showed that a redundant feature will always remain redundant at subsequent steps of the feature selection process, while an irrelevant feature may later turn into relevant. An important practical consequence is that redundant features, once detected, may be removed from the set of candidate features.

We derive upper and lower bounds for the target objective function and relate these bounds with the feature types. In particular, we showed that fully relevant features reach the maximum of the target objective function, irrelevant and redundant features reach the minimum, and relevant features take a value in between. This framework (target objective function, feature types, and objective function values for each feature type) provides a theoretical setting that can be used to compare the actual feature selection methods. Under this framework, the correct decisions at each step of the feature selection process are to select fully relevant features first and only afterwards relevant features, leave irrelevant features for future consideration (since they can later turn into relevant), and discard redundant features (since they will remain redundant).

Besides the theoretical framework, we defined a distributional setting, based on the definition of specific class, features, and a performance metric, designed to highlight the various deficiencies of methods. The setting includes four features, each belonging to one of the feature types defined above, and a class with two possible values. As performance metric, we introduced the minimum Bayes risk, a theoretical measure that does not rely on specific datasets and classifiers. The metric corresponds to the minimum total probability of misclassification for a certain class and set of selected features.

Actual feature selection methods are based on approximations of the target objective function. The target objective function comprises three terms, expressing the association between the candidate feature and the class (the relevance), the association between the candidate feature and the already selected features (the inter-feature redundancy), and the association between the candidate feature and the already selected features given the class (the class-relevant redundancy). The class-relevant redundancy is sometimes coined as the \textit{good} redundancy, since it expresses the contribution of the candidate feature to the explanation of the class, when taken together with already selected features. We also say that this term reflects the \textit{complementarity} between the candidate and the already selected features with respect to the class.

Method MIM was the first method to be proposed, and completely ignored redundancy. Methods MIFS, mRMR, and maxMIFS ignored complementary effects, i.e. they did not include the class-relevant redundancy term in their objective functions. These methods lose both the upper and lower bounds of the target objective function and, more importantly, lose the connection between the bounds and the specific feature types, i.e. it is no longer possible to guarantee that fully relevant and relevant features are selected before redundant and irrelevant features, or that fully relevant come before relevant.

Methods CIFE, JMI, CMIM, and JMIM considered complementarity effects, but in different ways. The main difference between these methods lies in the approximation of the redundancy terms (the ones related with inter-feature and class-relevant redundancies). These terms depend on the complete set of already selected features and are difficult to estimate. To overcome this difficulty, the methods approximate the redundancy terms by a function of already selected features taken individually. In particular, CIFE uses the sum of the associations with the individual already selected features, JMI uses the average, and both CMIM and JMIM use the maximum. In relation to other methods, JMIM introduced an extra term in its objective function, which is unimportant and leads to confusion in the selection process. The approximations of the remaining methods lead to the following problems: CIFE overscales the redundancy, JMI undervalues the redundancy, and CMIM undervalues the redundancy in a lower extent than JMI but penalizes the complementarity. The consequences of these approximations are that CIFE loses both the upper and lower bound of the target objective function, JMI and CMIM preserve only the lower bound, and JMIM preserves only the upper bound. Moreover, as in the case of the methods that ignore complementary, the methods lose the connection between the bounds of the target objective function and the specific feature types, except in a specific case for CMIM. The drawbacks of the various methods were summarized in Table \ref{tab:drawback}.


These results show that, for all methods, it is always possible to find cases where incorrect decisions are produced, and we have provided several examples throughout the paper and as part of our distributional setting. However, the drawbacks of the methods have different degrees of severity. MIM is a very basic method that we only considered for reference purposes. Ignoring complementary is a severe drawback that can lead to gross errors in the selection process. Thus, MIFS, mRMR, and maxMIFS, should be avoided. Regarding the methods that include complementarity effects, CIFE and JMIM should also be avoided, CIFE because its objective function is unbounded both inferiorly and superiorly due to the overscaled redundancy drawback, and JMIM because its objective function includes a bad approximation of an unimportant term that leads to confusion. Thus, the methods that currently have superior performance are JMI and CMIM. There is no clear-cut decision between these two methods, since both present drawbacks of equivalent degree of severity. JMI undervalues both the inter-feature and the class-relevant redundancy. CMIM also undervalues both types of redundancy. However, it tends to approximate better the inter-feature redundancy, but worse the class-relevant redundancy due to the problem of complementarity penalized.

%% file: appendix.tex
\appendix

\section{Computation of the terms in the tables of Section \ref{subsec:requiredquantities}}
\label{app:A}

In this section, we derive the expressions required for completing the tables given in Section \ref{subsec:requiredquantities}.

\subsection{Values in Table \ref{tab:entropyvariables}}
\label{app:A1}

\begin{paragraph}{Univariate differential entropies (continuous features).}

The entropy of $X$ is obtained from Example \ref{ex:normal}, in Section \ref{subsec:diffentropy}, considering $n=1$. As for the entropy of the other continuous feature, $X-k'Y$, the same expression can be used since it is widely known that linear or affine combinations of independent univariate features following normal distributions also follow normal distributions. All we need are the variances of these two features. The variance of $X$ is $1$ and the variance of $X-k'Y$ is $1+k^{'2}$. Therefore, the corresponding entropies are $h(X)=\frac{1}{2}\ln(2\pi e)$ and $h(X-k'Y)=\frac{1}{2}\ln(2\pi e (1+k^{'2}))$, respectively.

\end{paragraph}

\begin{paragraph}{Univariate entropies (discrete features and class).}

Concerning $Z$, applying Definition \ref{def:entropydiscrete}, $H(Z)=\ln(2)$.

We now discuss the values of $H(X_{{\rm disc}})$ and $H(C_k)$. Given that $X$ and $Y$ are independent and individually follow standard normal distributions, the joint density function of $(X,Y)$ is
\begin{equation*}
f_{X,Y}(x,y)=\frac{1}{2\pi}\exp(-(x^2+y^2))=\phi(x)\phi(y).
\end{equation*}
Therefore, the density at point $(x,y)$ only depends on the distance from this point to the origin, $\sqrt{x^2+y^2}$, in the context of the two-dimensional space defined by $(X,Y)$. As a consequence, the probability of $(X,Y)$ taking values in a region limited by two rays having the origin as starting point is given by $\alpha/(2\pi)$, with $\alpha$ denoting the angle between the two rays, as illustrated in Figure \ref{fig:symmetryorigin}. The circle is dashed in the figure since we can consider an infinite radius. 

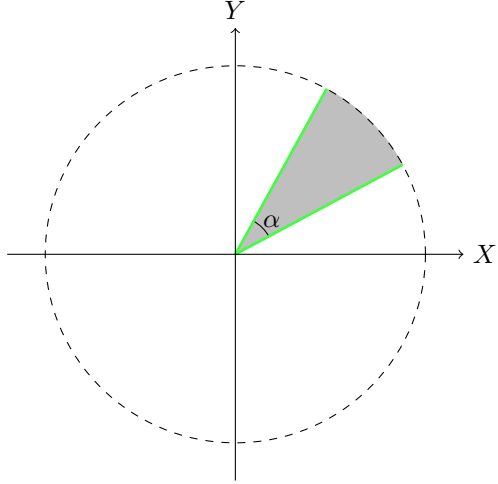
\begin{figure}[t!]
\centering
\begin{tikzpicture}[domain=0:3]
\filldraw[gray!50] (0,0) -- (2.2,1.187434) arc (0.4949340*57.2958:1.070142*57.2958:2.5) -- cycle;
\draw[dashed] (0,0) circle (2.5cm) node[right] {};
\draw (0.44,0.2374868) arc (0.4949340*57.2958:1.070142*57.2958:0.5) node[right] {$\alpha$};
\draw[green!75,line width=1.0pt] (0,0) -- (1.2,2.193171);
\draw[green!75,line width=1.0pt] (0,0) -- (2.2,1.187434);
\draw[->] (-3,0) -- (3,0) node[right] {$X$};
 \draw[->] (0,-3) -- (0,3) node[above] {$Y$};
\end{tikzpicture}
\caption{Angle between two rays starting from the origin; $\alpha/(2\pi)$ is the probability that $(X,Y)$ belongs to the region delimited by the two rays, when $X$ and $Y$ are two independent random variables following standard normal distribution.}
\label{fig:symmetryorigin}
\end{figure}

Considering an infinite radius is in fact what we need. Both $C_k$ and $X_{{\rm disc}}$ are characterized by a partition of $\mathbb{R}^2$ in two regions separated by a line that crosses the origin. Therefore, each region covers an angle $\alpha=\pi$, so that each region has associated probability $1/2$. Thus, $C_k$ and $X_{{\rm disc}}$ follow a Bernoulli distribution with success probability $1/2$, just as $Z$. As a result, $H(C_k)=H(X_{{\rm disc}})=H(Z)=\ln(2)$.

\end{paragraph}

\subsection{Values in Table \ref{tab:miinputclass}}
\label{app:A2}

\begin{paragraph}{Continuous features.}

In order to derive ${\rm MI}(C_k,X)$ and ${\rm MI}(C_k,X-k'Y)$, expression \eqref{eq:micondpropertymix} should be, in general, preferred over \eqref{eq:micondpropertymix2}. In fact, the entropy of a feature given the class can be obtained through the corresponding probability density functions; recall \eqref{eq:entropycontinuous}. These, in turn, are possible to derive easily in our setting. Therefore, we calculate the MI of interest using the representation
\begin{equation}
{\rm MI}(X_i,C_k)= h(X_i)-\sum_{j=0}^1 \int f_{X_i|C_k=j}(u) P(C_k=j) \ln f_{X_i|C_k=j}(u) du,
\label{eq:miwithclasscont}
\end{equation}
where $P(C_k=j)=1/2$, $j=0,1$. It all comes down to determining $f_{X_i|C_k=j}(u)$, $j=0,1$, as $h(X_i)$ is known (vide Table \ref{tab:entropyvariables}).

Given that the features follow a normal distribution, the conditional distribution of interest is the well-known skew-normal distribution \citep[vide][Ch. 5]{claudia}. Therefore, we only need to determine, in each case, the parameters of the mentioned distribution; recall \eqref{eq:skew}.

In the case of ${\rm MI}(C_k,X)$, it was proved \citep[Ch. 5]{claudia} that $X|C_k=j$, $j=0,1$, follow skew-normal distributions with parameters $(0,1,\frac{(-1)^{j+1}}{k})$; i.e. $X|C_k=j \sim {\rm SN}(0,1,\frac{(-1)^{j+1}}{k})$, $j=0,1$.

As for ${\rm MI}(C_k,X-k'Y)$, we use the procedure used for the determination of ${\rm MI}(C_k,X)$ \citep[see][Ch. 5]{claudia} to prove that $X-k'Y|C_k=j\sim {\rm SN}(0,\sqrt{1+k'^2},(-1)^{j+1}(\frac{1-kk'}{k+k'}))$, $j=0,1$. The procedure consists of obtaining the conditional distribution functions of the feature given the two different possible values of the class, taking then the corresponding derivatives in order to obtain the associated probability density functions.

In this context, we will need the probability density function $f_{X+kY,X-k' Y}(z,w)$. This can be obtained from the joint density of the pair $(X,Y)$. In fact, there is a way to obtain the probability density function of $g(X,Y)$, with $g$ being a bijective function, from the probability density function of $(X,Y)$, using the general well-known expression \citep[Ch. 2]{MR1231974}
\begin{equation}
f_{g(X,Y)}(z,w)=f_{X,Y}(g^{-1}(z,w))\left|\frac{dg^{-1}(z,w)}{d(z,w)}\right|,
\label{eq:inversepdf}
\end{equation}
where $|\frac{dg^{-1}(z,w)}{d(z,w)}|$ denotes the absolute value of the Jacobian of the inverse of the function $g$.

As for the inverse function of the transformation $g(X,Y)=(X+kY,X-k'Y)$, it is given by
\begin{equation*}
g^{-1}(z,w)=\left(\frac{kz+k' w}{k+k'},\frac{w-z}{k+k'}\right).
\end{equation*}
The absolute value of its Jacobian is $|1/(k+k')|$. As both $k$ and $k'$ are non-negative, this can be simply written as $1/(k+k')$. 

As a result, we have
\begin{equation*}
f_{X+kY,X-k' Y}(z,w)=\frac{1}{k+k'} \phi\left(\frac{kz+k' w}{k+k'}\right)\phi\left(\frac{w-z}{k+k'}\right), \quad (z,w)\in \mathbb{R}^2.
\end{equation*}

We can now proceed with the derivation of the distribution functions of interest. From now on, the distribution function of $\boldsymbol{Z}$ at $\boldsymbol{z}$ will be represented by $F_{\boldsymbol{Z}}(z)$.

We start with the case $C_k=0$:
\begin{align*}
&F_{X-k' Y|X+kY<0}(u)= P(X-k' Y \leq u|X+kY<0) \\
 &= \frac{P(X-k' Y\leq u,X+kY< 0)}{P(X+kY< 0)}\\
 &= 2 \int_{-\infty}^{0} \int_{-\infty}^{u} f_{X+kY,X-k' Y}(z,w) dw\, dz \\
 &= 2 \int_{-\infty}^{0} \int_{-\infty}^{u} \frac{1}{k+k'} \phi(\frac{kz+k' w}{k+k'})\phi(\frac{w-z}{k+k'}) dw\, dz\\
&=2 \int_{-\infty}^{0} \int_{-\infty}^{u} \frac{1}{k+k'}\frac{1}{2\pi}\exp\left\{-\frac{1}{2}\frac{(k^2+1)z^2}{(k+k')^2}\right\} \exp\left\{-\frac{1}{2}(k'^2+1)[w-\frac{z(1-kk')}{1+k'^2}]^2+\right.\\
& \hspace*{2.0cm} \left. z^2[\frac{(1+k^2)(1+k'^2)-(1-kk')^2}{1+k'^2}]\right\}dw\, dz\\
&=\int_{-\infty}^{u} \frac{1}{\sqrt{\pi}}\exp\left\{-\frac{1}{2}\frac{(k+k')^2 z^2}{(k+k')^2}\right\} \int_{-\infty}^{0} \frac{1}{\sqrt{\pi}(k+k')} \times \\
& \hspace*{2.0cm}\exp\{-\frac{1}{2}\frac{(1+k'^2)(w-\frac{z(1-kk')}{1+k'^2})^2}{(k+k')^2}\} dz\, dw\\
&=\sqrt{2} \int_{-\infty}^{u} \frac{1}{\sqrt{2\pi}}\frac{\sqrt{2}}{\sqrt{1+k'^2}} \int_{-\infty}^{0} \frac{1}{\frac{\sqrt{2\pi}(k+k')}{\sqrt{1+k'^2}}} \exp\left\{-\frac{1}{2}\frac{z^2(k+k')^2}{1+k'^2}\right\} dz\, dw\\
&=\frac{\sqrt{2}}{\sqrt{1+k'^2}} \int_{-\infty}^{u} \frac{1}{\sqrt{2\pi}{\sqrt{1+k'^2}}(k+k')} \exp\left\{-\frac{1}{2}\frac{z^2}{1+k'^2}\right\} \Phi(-\frac{(1-kk')z}{(k-k')\sqrt{1+k'^2}}) dz\\
&=\int_{-\infty}^{u} \frac{2}{\sqrt{1+k'^2}}\phi(\frac{z}{\sqrt{1+k'^2}}) \Phi(-\frac{(1-kk')z}{(k-k')\sqrt{1+k'^2}}) dz.
\end{align*}
Hence, $X-k'Y|X+kY<0\sim {\rm SN}(0,\sqrt{1+k'^2},\frac{-(1-kk')}{k+k'})$.

Some auxiliary steps were required in the first step of the derivation above in which $\frac{1}{k+k'} \phi(\frac{kz+k' w}{k+k'})\phi(\frac{z-w}{k-k'})$ was transformed significantly. The main technicality about such steps was the algebraic manipulation
\begin{align*}
&(kz+k'w)^2+(w-z)^2\\
&=(1+k^2)z^2-2wz(1-kk')+(k'^2+1)w^2\\
&=(1+k'^2)\{w^2-2w\frac{z(1-kk')}{1+k'^2}+[\frac{z(1-kk')}{1+k'^2}]^2- [\frac{z(1-kk')}{1+k'^2}]^2\}+(1+k^2)z^2\\
&=(1+k'^2)\{[w-\frac{z(1-kk')}{1+k'^2}]^2-\frac{[z(1-kk')]^2}{(1+k'^2)^2}\}+ (1+k^2)z^2\\
&=(1+k'^2)[w-\frac{z(1-kk')}{1+k'^2}]^2+z^2\frac{(k+k')^2}{1+k'^2}.
\end{align*}

As for the conditional case in which $C_k=1$, we provide a briefer version of the computation as most steps are the same as for $C_k=0$.
\begin{align*}
&F_{X-k' Y|X+kY\geq 0}(u) \\
 &= P(X-k' Y \leq u|X+kY\geq 0)\\
 &= \frac{P(X-k' Y\leq u,X+kY\geq 0)}{P(X+kY\geq 0)}\\
 &= 2 \int_{0}^{+\infty} \int_{-\infty}^{u} f_{X+kY,X-k' Y}(z,w) dw\, dz \\
 &= 2 \int_{0}^{+\infty} \int_{-\infty}^{u} \frac{1}{k+k'} \phi(\frac{kz+k' w}{k+k'})\phi(\frac{w-z}{k+k'}) dw\, dz\\
 &=2 \int_{-\infty}^{0} \int_{-\infty}^{u} \frac{1}{k+k'}\frac{1}{2\pi}\exp\left\{-\frac{1}{2}\frac{(k^2+1)z^2}{(k+k')^2}\right\} \exp\left\{-\frac{1}{2}(k'^2+1)[w-\frac{z(1-kk')}{1+k'^2}]^2+\right.\\
 & \hspace*{1.5cm} \left. z^2[\frac{(1+k^2)(1+k'^2)-(1-kk')^2}{1+k'^2}]\right\} dw\, dz\\
&=\int_{-\infty}^{u} \frac{2}{\sqrt{1+k'^2}}\phi(\frac{z}{\sqrt{1+k'^2}}) [1-\Phi(-\frac{(1-kk')z}{(k-k')\sqrt{1+k'^2}})] dz.
\end{align*}
Given the symmetry of the normal distribution, we know that $(1-\Phi(-x))=\Phi(x)$. This allows reducing the expression to
\begin{equation*}
\int_{-\infty}^{u} \frac{2}{\sqrt{1+k'^2}}\phi(\frac{z}{\sqrt{1+k'^2}}) \Phi(\frac{(1+kk')z}{(k-k')\sqrt{1+k'^2}}) dz.
\end{equation*}
Thus, $X-k' Y|X+kY\geq 0\sim {\rm SN}\left(0,\sqrt{1+k'^2},\frac{(1-kk')}{k+k'}\right)$. 

\end{paragraph}

\begin{paragraph}{Discrete features.}

In order to obtain ${\rm MI}(C_k,X_{{\rm disc}})$, we use \eqref{eq:micondproperty2}. We have ${\rm MI}(C_k,X_{{\rm disc}})=H(C_k)+H(X_{{\rm disc}})-H(C_k,X_{{\rm disc}})$. We only need to compute $H(C_k,X_{{\rm disc}})$ since the required univariate entropies are known from Table \ref{tab:entropyvariables}. From Definition \ref{def:entropydiscrete}, this requires obtaining the probabilities of the four possible combinations of values associated with the pair $(C_k,X_{{\rm disc}})$. We represent the regions associated with such values in Figure \ref{fig:midisc}, considering the two-dimensional space defined by the pair $(X,Y)$. Considering the reasoning used to obtain $H(C_k)$ and $H(X_{{\rm disc}})$, associated with Figure \ref{fig:symmetryorigin}, we only need the four angles covered by the associated four regions in order to compute their corresponding probabilities. The determination of such angles only requires the knowledge of $\theta$, represented in Figure \ref{fig:midisc} since the remaining angles consist of its supplementary, its opposite, and the opposite of its supplementary. 

\begin{figure}[t!]
\centering
\begin{tikzpicture}[domain=0:3]
\filldraw[gray!95](-3,3)--(3,3)--(3,-3)--(-3,-3)--cycle;
\filldraw[gray!65](-1.5,3)--(0,0)--(0,-3)--(-3,-3)--(-3,3)--cycle;
\filldraw[gray!35](0,0)--(0,-3)--(1.5,-3)--cycle;
\filldraw[gray!5](0,0)--(0,3)--(-1.5,3)--cycle;
\node[] at (-1.5,-0.8){$\boldsymbol{(c)}$};
\node[] at (0.5,-2.0){$\boldsymbol{(b)}$};
\node[] at (-0.5,2.0){$\boldsymbol{(a)}$};
\node[] at (1.5,0.8){$\boldsymbol{(d)}$};
\draw[smooth,thick,color=green,domain=-1.5:1.5]plot(\x,{-2*\x});
\draw (0,1.0) arc (90:90+26.565:1.0) node[above,right] {};
\node[] at (-0.25,1.2){$\theta$};
\draw[->] (-3,0) -- (3,0) node[right] {$X$};
 \draw[->] (0,-3) -- (0,3) node[above] {$Y$};
\draw[smooth,thick,color=blue] (0,3) -- (0,-3) ;
\end{tikzpicture}
\caption{Four regions associated with the joint distribution of $(X_{{\rm disc}},C_k)$, where $\theta=\arctan k$, when $X$ and $Y$ are two independent random variables following standard normal distribution.}
\label{fig:midisc}
\end{figure}
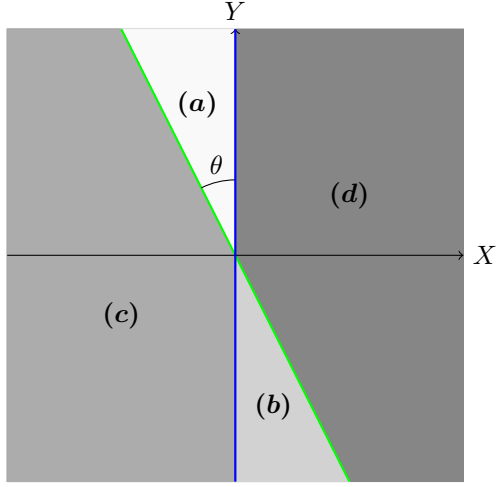

In the end, there are two angles (associated with regions $(a)$ and $(b)$ in Figure \ref{fig:midisc}) whose value is $\arctan k$, while the other two (associated with regions $(c)$ and $(d)$ in Figure \ref{fig:midisc}) have the value $(\pi-\arctan k)$, implying that
\begin{equation}
P(X_{{\rm disc}}=u,C_k=j) = \left\{ \begin{array}{cl}
 \frac{\pi-\arctan{k}}{2\pi},\!\!\!\! &\mbox{$u=0,j=0$ and $u=1,j=1$} \\
  \frac{\arctan{k}}{2\pi},\!\!\!\! &\mbox{$u=0,j=1$ and $u=1,j=0$}
       \end{array}. \right.
       \label{eq:jointdiscprobs}
\end{equation}
As a result,
\begin{equation*}
H(C_k,X_{{\rm disc}})=-2\times\left(\frac{\arctan k}{2\pi}\ln(\frac{\arctan k}{2\pi})+(\frac{1}{2}-\frac{\arctan{k}}{2\pi})\ln(\frac{1}{2}-\frac{\arctan{k}}{2\pi})\right).
\end{equation*}

We obtain
\begin{equation*}
{\rm MI}(C_k,X_{{\rm disc}})=2\ln(2)+\frac{\arctan{k}}{\pi}\ln(\frac{\arctan{k}}{2\pi})+(1-\frac{\arctan{k}}{\pi})\ln(\frac{1}{2}-\frac{\arctan{k}}{2\pi}).
\end{equation*}

As for ${\rm MI}(C_k,Z)$, its value is 0 since $C_k$ and $Z$ are independent.

\end{paragraph}

\subsection{Values in Table \ref{tab:miinputinput}}
\label{app:A3}

We start with ${\rm MI}(X,X_{{\rm disc}})$. $X_{{\rm disc}}$ is redundant given $\{X\}$, so that $H(X_{{\rm disc}}|X)=0$. Additionally, $H(X_{{\rm disc}})=\ln(2)$ (vide Table \ref{tab:entropyvariables}), so that, by \eqref{eq:micondproperty}, ${\rm MI}(X,X_{{\rm disc}})=H(X_{{\rm disc}})-H(X_{{\rm disc}}|X)=\ln(2)-0=\ln(2)$. 

As for ${\rm MI}(X-k'Y,X_{{\rm disc}})$, we first note that the deduction of ${\rm MI}(C_k,X)$ \citep[Ch. 5]{claudia} can be similarly done for the case where $k$ in \eqref{eq:class} is allowed to be negative, from which one would conclude that ${\rm MI}(C_{-k'},X)={\rm MI}(C_{k'},X)$. In turn, $X_{{\rm disc}}$ and $X-k'Y$ are obtained from $C_{k'}$ and $X$, respectively, by rotating $k'$ degrees anticlockwise, considering the two-dimensional space defined by the pair $(X,Y)$. As a result, ${\rm MI}(X_{{\rm disc}},X-k'Y)={\rm MI}(C_{k'},X)$. In fact, MI is invariant under one-to-one transformations \citep[see][for more details]{dadkhah2010performance}. Finally, by transitivity, ${\rm MI}(X-k'Y,X_{{\rm disc}})={\rm MI}(C_{k'},X)$.

As for ${\rm MI}(X,X-k'Y)$, we use \eqref{eq:micondproperty2}. We have ${\rm MI}(X,X-k'Y)=h(X)+h(X-k'Y)-h(X,X-k'Y)$. We only need to compute $h(X,X-k'Y)$ since the univariate entropies are known (vide Table \ref{tab:entropyvariables}). As both features follow normal distributions, the joint distribution is a bivariate normal distribution, whose entropy depends on the determinant of the covariance matrix, as described in Example \ref{ex:normal}. The value of the mentioned determinant is $k'^2$, so that ${\rm MI}(X,X-k'Y)=\frac{1}{2}\ln\left(1+1/k'^2\right)$. 

The MI terms involving $Z$ do not require any calculation given that $Z$ is independent of $X$, $X_{{\rm disc}}$, and $X-k'Y$, implying that ${\rm MI}(Z,X)={\rm MI}(Z,X_{{\rm disc}})={\rm MI}(Z,X-k'Y)=0$. 

\subsection{Values in Table \ref{tab:miinputinputcond}}
\label{app:A4}

For deriving ${\rm MI}(X,X_{{\rm disc}}|C_k)$, we use \eqref{eq:miconddiscrete}. We have ${\rm MI}(X,X_{{\rm disc}}|C_k)=H(X_{{\rm disc}}|C_k)-H(X_{{\rm disc}}|X,C_k)=0$. Noting that $H(X_{{\rm disc}}|X,C_k)=0$ since $H(X_{{\rm disc}}|X)=0$, using \eqref{eq:infodonthurtgen}, we have ${\rm MI}(X,X_{{\rm disc}}|C_k)=H(X_{{\rm disc}}|C_k)$. In turn, $H(X_{{\rm disc}}|C_k)=H(X_{{\rm disc}})-{\rm MI}(X_{{\rm disc}},C_k)$, where the values $H(X_{{\rm disc}})$ and ${\rm MI}(X_{{\rm disc}},C_k)$ have been obtained; recall Tables \ref{tab:entropyvariables} and \ref{tab:miinputinput}, respectively. We conclude that 
\begin{equation*}
{\rm MI}(X,X_{{\rm disc}}|C_k)=-\frac{\arctan{k}}{\pi}\ln(\frac{\arctan{k}}{\pi})-(1-\frac{\arctan{k}}{\pi})\ln(1-\frac{\arctan{k}}{\pi}).
\end{equation*}

Concerning ${\rm MI}(X,X-k'Y|C_k)$, we use \eqref{eq:miconddiscrete3}. We have ${\rm MI}(X,X-k'Y|C_k)=h(X|C_k)+h(X-k'Y|C_k)-h(X,X-k'Y|C_k)$. The first two terms were obtained in the context of the determination of ${\rm MI}(X,C_k)$ and ${\rm MI}(X-k'Y|C_k)$; they consist of the second term in \eqref{eq:miwithclasscont}.

As in the computation of class-conditional entropies for univariate continuous features, recall \eqref{eq:miwithclasscont}, we only need the joint probability density functions of $(X,X-k'Y)$ conditioned on the two possible values of the class to determine $h(X,X-k'Y|C_k)$. These are obtained from the corresponding conditional probability density functions associated with the pair $(X,Y)$, using \eqref{eq:inversepdf}.

We first need to derive the probability density functions associated with the distributions $(X,Y)|C_k=j$, $j=0,1$. These are obtained starting from the corresponding distribution functions, as done in the context of the determination of the probability density functions needed in sequence of \eqref{eq:miwithclasscont}.

We start with $C_k=1$. We have
\begin{align*}
F_{(X,Y)|C_k=1}(x,y) =& \frac{P(X\leq x,Y\leq y,X+kY\geq 0)}{P(X+kY\geq 0)}\\
=& 2 P(X\leq x,Y\leq y,X+kY\geq 0).
\end{align*}

In order to proceed, we separate the expression in two cases. In fact, if $x<-ky$, the value is simply 0. If, instead, $x\geq -ky$, we have
\begin{align*}
 2 P(X\leq x,Y\leq y,X+kY\geq 0)=& 2 \int_{-\infty}^{x} \int_{-\frac{u}{k}}^{y} \phi(u)\phi(v) dv\, du\\
=& 2 \int_{-\infty}^{x} \phi(u)[\Phi(v)-\Phi(-\frac{u}{k})] du\\
=& 2 \Phi(y)\Phi(x) - F_{{\rm SN}(0,1,-\frac{1}{k})}(x).
\end{align*}

Thus, the corresponding density is given by
\begin{equation*}
f_{(X,Y)|C_k=1}(x,y) = \left\{ \begin{array}{cl}
2 \phi(y)\phi(x), &\mbox{$x\geq -ky$} \\
0, &\mbox{$x< -ky$}
       \end{array}. \right.
\end{equation*}

We now make the same type of deduction for $C_k=0$. We have
\begin{align*}
F_{(X,Y)|C_k=0}(x,y) =& \frac{P(X\leq x,Y\leq y,X+kY< 0)}{P(X+kY< 0)}\\
=& 2 P(X\leq x,Y\leq y,X+kY< 0).
\end{align*}

If $x<-ky$, we obtain
\begin{align*}
 2 P(X\leq x,Y\leq y,X+kY< 0)=& 2 \int_{-\infty}^{x} \int_{-\infty}^{y} \phi(u)\phi(v) dv\, du\\
=& 2 \Phi(y)\Phi(x).
\end{align*}

If, instead, $x\geq -ky$, we have
\begin{align*}
2 P(X\leq x,Y\leq y,X+kY\geq 0)=& 2 \left[\int_{-\infty}^{x} \int_{-\infty}^{y} \phi(u)\phi(v) du\, dv - \int_{-\infty}^{x} \int_{-\frac{u}{k}}^{y} \phi(u)\phi(v) dv\, du\right]\\
=& 2 \Phi(y)\Phi(x)-[2\Phi(y)\Phi(x)-F_{{\rm SN}(0,1,-\frac{1}{k})}(x)]\\
=& F_{{\rm SN}(0,1,-\frac{1}{k})}(x).
\end{align*}

Thus, the corresponding density is given by
\begin{equation*}
f_{(X,Y)|C_k=0}(x,y) = \left\{ \begin{array}{cl}
 0, &\mbox{$x\geq -ky$} \\
2 \phi(y)\phi(x), &\mbox{$x< -ky$}
       \end{array}. \right.
\end{equation*}

We can now obtain $f_{X,X-k' Y}(u,v)$ using \eqref{eq:inversepdf}. In this case, $g^{-1}(z,w)=(z,\frac{z-w}{k'})$. The determinant of the corresponding Jacobian is $1/k'$. Therefore, we have $f_{(X,X-k' Y)|C_k=c}(u,v)=f_{(X,Y)|C_k=c}(u,\frac{u-v}{k'})\frac{1}{k'}$, $c=0,1$. Thus, we obtain
\begin{equation*}
f_{(X,X-k'Y)|C_k=1}(u,v) = \left\{ \begin{array}{cl}
 2\phi(u)\phi(\frac{v-u}{k'})\frac{1}{k'},\!\!\!\! &\mbox{$u> \frac{k}{k'+k}v$} \\
 0,\!\!\!\! &\mbox{$u\leq \frac{k}{k'+k}v$}
       \end{array} \right.
\end{equation*}
and
\begin{equation*}
f_{(X,X-k'Y)|C_k=0}(u,v) = \left\{ \begin{array}{cl}
 0,\!\!\!\! &\mbox{$u> \frac{k}{k'+k}v$} \\
 2\phi(u)\phi(\frac{v-u}{k'})\frac{1}{k'},\!\!\!\! &\mbox{$u\leq \frac{k}{k'+k}v$}
       \end{array}, \right.
\end{equation*}
where we note that $\phi(\frac{v-u}{k'})\frac{1}{k'}$ can be also seen as the density for a normal distribution with parameters $(u,k'^2)$ evaluated at $v$.

We have
\begin{equation*}
h(X,X-k'Y|C_k=1)=-\int_{-\infty}^{+\infty} \int_{\frac{k}{k'+k}v}^{+\infty} 2\phi(u)\phi(\frac{v-u}{k'})\frac{1}{k'} \ln(2\phi(u)\phi(\frac{v-u}{k'})\frac{1}{k'}) dv\, du
\end{equation*}
and 
\begin{equation*}
h(X,X-k'Y|C_k=0)=-\int_{-\infty}^{+\infty} \int_{-\infty}^{\frac{k}{k'+k}v} 2\phi(u)\phi(\frac{v-u}{k'})\frac{1}{k'} \ln(2\phi(u)\phi(\frac{v-u}{k'})\frac{1}{k'}) dv\, du.
\end{equation*}

This implies that
\begin{equation*}
h(X,X-k'Y|C_k)=-\int_{-\infty}^{+\infty} \int_{-\infty}^{+\infty} \phi(u)\phi(\frac{v-u}{k'})\frac{1}{k'} \ln(2\phi(u)\phi(\frac{v-u}{k'})\frac{1}{k'}) dv\, du.
\end{equation*}

The part inside the logarithm can be re-written considering the explicit expression of the density of a standard normal distribution. We have
\begin{align*}
\ln(2\phi(u)\phi(\frac{v-u}{k'})\frac{1}{k'})&=\ln(\frac{1}{\pi k'}\exp(-\frac{1}{2k'^2}[(1+k')u^2-2uv+v^2]))\\
&=-\ln(\pi k')-\frac{1}{2k'^2}[(1+k')u^2-2uv+v^2].
\end{align*}

We can still write this as
\begin{equation*}
h(X,X-k'Y|C_k)=\int_{-\infty}^{+\infty} \int_{-\infty}^{+\infty} \phi(u)\phi(v) [\ln(\pi k')+\frac{1}{2k'^2}[(1+k')u^2-2uv+v^2]] dv\, du.
\end{equation*}
The final result is $\ln(\pi k')+\frac{1}{2}+\frac{1}{2}=\ln{\pi}+\ln{k'}+1$. In fact, the first term is a double integral in the whole space of a probability density function, associated with $(X,Y)$, times a constant, so that the result is such constant, $\ln(\pi k')$; while the remaining terms are also easy to obtain since they consist of constants multiplied with first and second order moments from the normal distribution with parameters $(u,k')$.

As for ${\rm MI}(X-k' Y,X_{{\rm disc}}|C_k)$, we compute it, using \eqref{eq:miconddiscrete}, through its representation ${\rm MI}(X-k' Y,X_{{\rm disc}}|C_k)=h(X-k' Y|C_k)-h(X-k'Y|X_{{\rm disc}},C_k)$. Note that $h(X-k' Y|C_k)$ has already been derived, in the context of the determination of ${\rm MI}(X-k'Y,C_k)$; recall \eqref{eq:miwithclasscont}. Thus, we only need to derive $h(X-k'Y|X_{{\rm disc}},C_k)$.

We need to obtain $f_{X-k'Y|X_{{\rm disc}}=u,C_k=j}(v)$ for each possible combination of pairs $(u,j)$. We first note that
\begin{align*}
f_{X-k'Y|X_{{\rm disc}}=u,C_k=j}&(v)=\frac{d}{dv} F_{X-k'Y|X_{{\rm disc}}=u,C_k=j}(v)\\
&= \frac{d}{dv} \frac{P(X-k'Y\leq v,X_{{\rm disc}}=u,C_k=j)}{ P(X_{{\rm disc}}=u,C_k=j)}\\
&= \frac{1}{P(X_{{\rm disc}}=u,C_k=j)} \frac{d}{dv} P(X-k'Y\leq v,X_{{\rm disc}}=u,C_k=j).
\end{align*} 
The values $P(X_{{\rm disc}}=u,C_k=j)$, $u=0,1$ and $j=0,1$, can be found in \eqref{eq:jointdiscprobs}.

Therefore, we only need to obtain $P(X-k'Y\leq v,X_{{\rm disc}}=u,C_k=j)$.
We start with $u=1$ and $j=1$. We have to split in two cases. For $v\geq 0$, 
\begin{align*}
& P(X-k'Y\leq v,X_{{\rm disc}}=1,C_k=1)\\
&= \int_{-\frac{v}{k'+k}}^{0} \int_{-kz}^{v+k'z} \phi(w)\phi(z) dw\, dz+ \int_{0}^{+\infty} \int_{0}^{v+k'z} \phi(w)\phi(z) dz\, dw\\
&= \int_{-\frac{z}{k'+k}}^{0} \phi(z)[\Phi(v+k'z)-\Phi(-kz)] dz + \int_{0}^{+\infty} \phi(z)[\Phi(v+k'z)-\frac{1}{2}] dz\\
&= \int_{-\frac{v}{k'+k}}^{0} \phi(z)\Phi(v+k'z) dz - \frac{1}{2}[F_{{\rm SN}(0,1,-k)}(0)-F_{{\rm SN}(0,1,-k)}(-\frac{v}{k'+k})] \\
&\hspace*{1.0cm}+ \int_{0}^{+\infty} \phi(z)\Phi(v+k'z) dz-\frac{1}{4}\\
&= \int_{-\frac{v}{k'+k}}^{+\infty} \phi(z)\Phi(v+k'z) dz -\frac{1}{2}F_{{\rm SN}(0,1,-k)}(0)+\frac{1}{2}F_{{\rm SN}(0,1,-k)}(-\frac{v}{k'+k})-\frac{1}{4}.
\end{align*}
In turn, for $v<0$,
\begin{align*}
& P(X-k'Y\leq v,X_{{\rm disc}}=1,C_k=1)\\
&= \int_{-\frac{v}{k'}}^{+\infty} \int_{0}^{v+k'z} \phi(w)\phi(z) dz\, dw\\
&= \int_{-\frac{v}{k'}}^{+\infty} \phi(z)[\Phi(v+k'z)-\frac{1}{2}] dz\\
&= \int_{-\frac{v}{k'}}^{+\infty} \phi(z)\Phi(v+k'z) dz -\Phi(\frac{v}{k'}).
\end{align*}

We now need to take the derivative of the two expressions with respect to $v$ to obtain the corresponding conditional density functions. In the case of $v\geq 0$,
\begin{align*}
&\frac{d}{dv}\left[\int_{-\frac{v}{k'+k}}^{+\infty} \phi(z)\Phi(v+k'z) dz -\frac{1}{2}F_{{\rm SN}(0,1,-k)}(0)+\frac{1}{2}F_{{\rm SN}(0,1,-k)}(-\frac{v}{k'+k})-\frac{1}{4}\right]\\
&= \int_{-\frac{v}{k'+k}}^{+\infty} \phi(z)\phi(v+k'z) dz + \frac{1}{2}f_{{\rm SN}(0,1,-k)}(-\frac{v}{k'+k})\frac{1}{k'+k}\\
& \hspace*{1.0cm} -\frac{1}{2}f_{{\rm SN}(0,1,-k)}(-\frac{v}{k'+k})\frac{1}{k'+k}\\
&=\int_{-\frac{v}{k'+k}}^{+\infty} \phi(z)\phi(v+k'z) dz;
\end{align*}
while, for $v<0$, 
\begin{align*}
&\frac{d}{dv}\left[\int_{-\frac{v}{k'}}^{+\infty} \phi(z)\Phi(v+k'z) dz -\frac{1}{2}\Phi(\frac{v}{k'})\right]\\
&=\int_{-\frac{v}{k'}}^{+\infty} \phi(z)\phi(v+k'z) dz +\frac{1}{2}\phi(\frac{v}{k'})\frac{1}{k'}-\frac{1}{2}\phi(\frac{v}{k'})\frac{1}{k'}\\
&=\int_{-\frac{v}{k'}}^{+\infty} \phi(z)\phi(v+k'z) dz.
\end{align*}
Note that the following important result \citep[Ch. 3]{abramowitz1964handbook} was required in order to obtain both final expressions above:
\begin{equation}
\frac{d}{dx}\int_{a(x)}^{b(x)}g(x,y)dy=\int_{a(x)}^{b(x)}\frac{dg(x,y)}{dx}dy+g(x,b(x))b'(x)-g(x,a(x))a'(x).
\label{eq:derivint}
\end{equation}
This result was applied to $\frac{d}{dv} \int_{-\frac{v}{k'+k}}^{+\infty} \phi(z)\Phi(v+k'z) dz$, in the expression for $v\geq 0$, and also to $\frac{d}{dv} \int_{-\frac{v}{k'}}^{+\infty} \phi(z)\Phi(v+k'z) dz$, concerning the case $v<0$.

The desired probability density function is 
\begin{equation*}
f_{X-k'Y|X_{{\rm disc}}=1,C_k=1}(v) = \left\{ \begin{array}{cl}
\frac{2\pi}{\pi-\arctan{k}}\int_{-\frac{v}{k'+k}}^{+\infty} \phi(z)\phi(v+k'z) dz, &\mbox{$v\geq 0$}\\
\frac{2\pi}{\pi-\arctan{k}}\int_{-\frac{v}{k'}}^{+\infty} \phi(z)\phi(v+k'z) dz, &\mbox{$v<0$}
       \end{array}. \right.
\end{equation*}

We now consider $u=0$ and $j=1$. We have to split again in two cases. For $v\geq 0$, 
\begin{align*}
 P(X-k'Y\leq v,X_{{\rm disc}}=0,C_k=1)
=& \int_{0}^{+\infty} \int_{-kz}^{0} \phi(w)\phi(z) dz\, dw\\
=& \int_{0}^{+\infty} \phi(z)[\frac{1}{2}-\Phi(-kz)] dz\\
=& \frac{1}{2}F_{{\rm SN}(0,1,-k)}(0)-\frac{1}{4}.
\end{align*}
For $v<0$,
\begin{align*}
& P(X-k'Y\leq v,X_{{\rm disc}}=0,C_k=1)\\
&= \int_{-\frac{v}{k'+k}}^{-\frac{v}{k'}} \int_{-kz}^{v+k'z} \phi(w)\phi(z) dw\, dz + \int_{-\frac{v}{k'}}^{+\infty} \int_{-kz}^{0} \phi(w)\phi(z) dz\, dw\\
&= \int_{-\frac{v}{k'+k}}^{-\frac{v}{k'}} \phi(z)[\Phi(v+k'z)-\Phi(-kz)] dz +  \int_{-\frac{v}{k'}}^{+\infty} \phi(z)[\frac{1}{2}-\Phi(-kz)] dz\, dw\\
&= \int_{-\frac{v}{k'+k}}^{-\frac{v}{k'}} \phi(z)\Phi(v+k'z) dz-\frac{1}{2}[F_{{\rm SN}(0,1,-k)}(-\frac{v}{k'})-F_{{\rm SN}(0,1,-k)}(-\frac{v}{k'+k})]+\frac{1}{2}[1-\Phi(-\frac{v}{k'})]\\
&\hspace*{1.0cm}-\frac{1}{2}[1-F_{{\rm SN}(0,1,-k)}(-\frac{v}{k'})]\\
&= \int_{-\frac{v}{k'+k}}^{-\frac{v}{k'}} \phi(z)\Phi(v+k'z)dz+ \frac{1}{2}F_{{\rm SN}(0,1,-k)}(-\frac{v}{k'+k})+\frac{1}{2}\Phi(\frac{v}{k'}).
\end{align*}

We now need to take the derivative of the two expressions with respect to $v$ to obtain the corresponding conditional density functions. In the case of $v\geq 0$, it is simply $0$ since there is no dependency on $v$. As for $v<0$,
\begin{align*}
& \frac{d}{dv}\left[\int_{-\frac{v}{k'+k}}^{-\frac{v}{k'}} \phi(z)\Phi(v+k'z) dz + \frac{1}{2}F_{{\rm SN}(0,1,-k)}(-\frac{v}{k'+k})+\frac{1}{2}\Phi(\frac{v}{k'})\right]\\
&=\int_{-\frac{v}{k'+k}}^{-\frac{v}{k'}} \phi(z)\phi(v+k'z) dz + \frac{1}{2}f_{{\rm SN}(0,1,-k)}(-\frac{v}{k'+k})\frac{1}{k'+k}-\frac{1}{2}\phi(\frac{v}{k'})\frac{1}{k'}\\
&\hspace*{1.0cm}-\frac{1}{2}f_{{\rm SN}(0,1,-k)}(-\frac{v}{k'+k})\frac{1}{k'+k}+\frac{1}{2}\phi(\frac{v}{k'})\frac{1}{k'}\\
&=\int_{-\frac{v}{k'+k}}^{-\frac{v}{k'}} \phi(z)\phi(v+k'z) dz.
\end{align*}
Once again, \eqref{eq:derivint} was applied, in this case to $\int_{-\frac{v}{k'+k}}^{-\frac{v}{k'}} \phi(z)\Phi(v+k'z) dz$.

The desired probability density function is 
\begin{equation*}
f_{X-k'Y|X_{{\rm disc}}=0,C_k=1}(v) = \left\{ \begin{array}{cl}
 0, &\mbox{$v\geq 0$} \\
\frac{2\pi}{\arctan{k}} \int_{-\frac{v}{k'+k}}^{-\frac{v}{k'}} \phi(z)\phi(v+k'z) dz, &\mbox{$v<0$}
       \end{array}. \right.
\end{equation*}

We now consider $u=1$ and $j=0$. For $v\geq 0$, 
\begin{align*}
& P(X-k'Y\leq v,X_{{\rm disc}}=1,C_k=0)\\
&= \int_{-\frac{v}{k'}}^{-\frac{v}{k'+k}} \int_{0}^{v+k'z} \phi(w)\phi(z) dw\, dz + \int_{-\frac{v}{k'+k}}^{0} \int_{0}^{-kz} \phi(w)\phi(z) dz\, dw\\
&= \int_{-\frac{v}{k'}}^{-\frac{v}{k'+k}} \phi(z)[\Phi(v+k'z)-\frac{1}{2}] dz + \int_{-\frac{v}{k'+k}}^{0} \phi(z)[\Phi(-kz)-\frac{1}{2}] dz\, dw\\
&= \int_{-\frac{v}{k'}}^{-\frac{v}{k'+k}} \phi(z)\Phi(v+k'z) dz -\frac{1}{2}[\Phi(-\frac{v}{k'+k})-\Phi(-\frac{v}{k'})]\\
& \hspace*{2.0cm} +\frac{1}{2}[F_{{\rm SN}(0,1,-k)}(0)-F_{{\rm SN}(0,1,-k)}(-\frac{v}{k'+k})]-\frac{1}{2}\Phi(\frac{v}{k'+k})\\
&= \int_{-\frac{v}{k'}}^{-\frac{v}{k'+k}} \phi(z)\Phi(v+k'z) dz -\frac{1}{2}\Phi(\frac{v}{k'+k})\\
& \hspace*{2.0cm} +\frac{1}{2}F_{{\rm SN}(0,1,-k)}(0)-\frac{1}{2}F_{{\rm SN}(0,1,-k)}(-\frac{v}{k'+k}).
\end{align*}
As for the case $v<0$, $P(X-k'Y\leq v,X_{{\rm disc}}=1,C_k=0)=0$.

We now need to take the derivative with respect to $v$ of the expression obtained for $v\geq 0$ (the derivative of the one for $v< 0$ is 0) to obtain the corresponding conditional density functions. We have
\begin{align*}
&\frac{d}{dv}\left[\int_{-\frac{v}{k'}}^{-\frac{v}{k'+k}} \phi(z)\Phi(v+k'z) dz -\frac{1}{2}\Phi(-\frac{v}{k'+k})+ \frac{1}{2}F_{{\rm SN}(0,1,-k)}(0)- \frac{1}{2}F_{{\rm SN}(0,1,-k)}(-\frac{v}{k'+k})\right] \\
& = \int_{-\frac{v}{k'}}^{-\frac{v}{k'+k}} \phi(z)\phi(v+k'z) dz - \frac{1}{2}\phi(-\frac{v}{k'+k})-\frac{1}{2}f_{{\rm SN}(0,1,-k)}(-\frac{v}{k'+k})\frac{1}{k'+k}\\
& \hspace*{1.0cm} +\frac{1}{2}\phi(-\frac{v}{k'+k})+\frac{1}{2}f_{{\rm SN}(0,1,-k)}(-\frac{v}{k'+k})\frac{1}{k'+k}\\
&= \int_{-\frac{v}{k'}}^{-\frac{v}{k'+k}} \phi(z)\phi(v+k'z) dz.
\end{align*}
We again applied \eqref{eq:derivint}, in this case to $\int_{-\frac{v}{k'}}^{-\frac{v}{k'+k}} \phi(z)\Phi(v+k'z) dz$.

The desired probability density function is 
\begin{equation*}
f_{X-k'Y|X_{{\rm disc}}=1,C_k=0}(v) = \left\{ \begin{array}{cl}
\frac{2\pi}{\pi-\arctan{k}} \int_{-\frac{v}{k'}}^{-\frac{v}{k'+k}} \phi(z)\phi(v+k'z) dz, &\mbox{$v\geq 0$} \\
 0, &\mbox{$v<0$}
       \end{array}. \right.
\end{equation*}

We finally consider $u=0$ and $j=0$. For $v\geq 0$, 
\begin{align*}
& P(X-k'Y\leq v,X_{{\rm disc}}=0,C_k=0)\\
&= \int_{-\infty}^{0} \int_{-\infty}^{+\infty} \phi(w)\phi(z) dw\, dz - \int_{0}^{+\infty} \int_{-kz}^{0} \phi(w)\phi(z) dz\, dw \\
&\hspace*{1.0cm} -\int_{-\infty}^{-\frac{v}{k'}} \int_{v+k'z}^{0} \phi(w)\phi(z) dz\, dw\\
&= \frac{1}{2}- \int_{-\infty}^{0} \phi(z)\left[\Phi(-kz)-\frac{1}{2}\right] dz- \int_{-\infty}^{-\frac{v}{k'}} \left[\frac{1}{2}-\Phi(v+k'z)\right]\phi(z) dz\\
&= \frac{3}{4}-\frac{1}{2}F_{{\rm SN}(0,1,-k)}(0)-\frac{1}{2}\Phi\left(-\frac{v}{k'}\right)+ \int_{-\infty}^{-\frac{v}{k'}} \Phi(v+k'z)\phi(z) dz.
\end{align*}
As for the case $v<0$,
\begin{align*}
& P(X-k'Y\leq v,X_{{\rm disc}}=0,C_k=0)\\
&= \int_{-\infty}^{-\frac{v}{k'+k}} \int_{-\infty}^{v+k'z} \phi(w)\phi(z) dw\, dz + \int_{-\frac{v}{k'+k}}^{+\infty} \int_{-\infty}^{-kz} \phi(w)\phi(z) dz\, dw\\
&= \int_{-\infty}^{-\frac{v}{k'+k}} \Phi(v+k'z)\phi(z) dw\, dz+ \int_{-\frac{v}{k'+k}}^{+\infty} \Phi(-kz)\phi(z) dz\\
&= \int_{-\infty}^{-\frac{v}{k'+k}} \Phi(v+k'z)\phi(z)  dz+\frac{1}{2}- \frac{1}{2}F_{{\rm SN}(0,1,-k)}\left(-\frac{v}{k'+k}\right).
\end{align*}

We again need to take the derivative of the two expressions with respect to $v$ to obtain the corresponding conditional density functions. In the case of $v\geq 0$,
\begin{align*}
&\frac{d}{dv}\left[\frac{1}{4}+\frac{1}{2}F_{{\rm SN}(0,1,k)}(0)-\frac{1}{2}\Phi(-\frac{v}{k'})+ \int_{-\infty}^{\frac{v}{k'}} \Phi(v+k'z)\phi(z) dz\right]\\
&= \int_{-\infty}^{\frac{v}{k'}} \phi(v+k'z)\phi(z) dz - \frac{1}{2}\phi(-\frac{v}{k'})\frac{1}{k'} + \frac{1}{2}\phi(-\frac{v}{k'})\frac{1}{k'}\\
&=\int_{-\infty}^{\frac{v}{k'}} \phi(z)\phi(v+k'z) dz;
\end{align*}
while, for $v<0$, 
\begin{align*}
&\frac{d}{dv}\left[\int_{-\infty}^{-\frac{v}{k'+k}} \Phi(v+k'z)\phi(z) dz+\frac{1}{2}-\frac{1}{2}F_{{\rm SN}(0,1,-k)}(-\frac{v}{k'+k})\right]\\
& = \int_{-\infty}^{-\frac{v}{k'+k}} \phi(z)\phi(v+k'z) dz -\frac{1}{2}f_{{\rm SN}(0,1,-k)}(-\frac{v}{k'+k})\frac{v}{k'+k}\\
& \hspace*{1.0cm} +\frac{1}{2}f_{{\rm SN}(0,1,-k)}(-\frac{v}{k'+k})\frac{v}{k'+k}\\
&=\int_{-\infty}^{-\frac{v}{k'+k}} \phi(z)\phi(v+k'z) dz.
\end{align*}
We applied \eqref{eq:derivint} to $\int_{-\infty}^{\frac{v}{k'}} \Phi(v+k'z)\phi(z) dz$ and $\int_{-\infty}^{-\frac{v}{k'+k}} \Phi(v+k'z)\phi(z) dz$.

The desired probability density function is
\begin{equation*}
f_{X-k'Y|X_{{\rm disc}}=0,C_k=0}(v) = \left\{ \begin{array}{cl}
\frac{2\pi}{\pi-\arctan{k}} \int_{-\infty}^{\frac{v}{k'}} \phi(z)\phi(v+k'z) dz, &\mbox{$v\geq 0$}\\
\frac{2\pi}{\pi-\arctan{k}} \int_{-\infty}^{-\frac{v}{k'+k}} \phi(z)\phi(v+k'z) dz, &\mbox{$v<0$}
       \end{array}. \right.
\end{equation*}

We can finally obtain an expression for $h(X-k'Y|X_{{\rm disc}},C_k)$:
\begin{align}
\nonumber
-\frac{\pi-\arctan{k}}{2\pi}&\left(\int_{-\infty}^{0} (\int_{-\frac{v}{k'}}^{+\infty} \frac{2\pi}{\pi-\arctan{k}} \zeta(z,v) dz) \ln(\int_{-\frac{v}{k'}}^{+\infty} \frac{2\pi}{\pi-\arctan{k}} \zeta(z,v) dz) dv+\right.\\
\nonumber
& \hspace*{0.3cm}\int_{0}^{+\infty} (\int_{-\frac{v}{k'+k}}^{+\infty} \frac{2\pi}{\pi-\arctan{k}} \zeta(z,v) dz) \ln(\int_{-\frac{v}{k'+k}}^{+\infty} \frac{2\pi}{\pi-\arctan{k}} \zeta(z,v) dz) dv+\\
\nonumber
& \hspace*{0.3cm}\int_{-\infty}^{0} (\int_{-\infty}^{\frac{v}{k'+k}} \frac{2\pi}{\pi-\arctan{k}}\zeta(z,v) dz) \ln(\int_{-\infty}^{\frac{v}{k'+k}} \frac{2\pi}{\pi-\arctan{k}} \zeta(z,v) dz) dv+\\
\nonumber
& \hspace*{0.3cm}\left. \int_{0}^{+\infty} (\int_{-\infty}^{\frac{v}{k'}} \frac{2\pi}{\pi-\arctan{k}} \zeta(z,v) dz) \ln(\int_{-\infty}^{\frac{v}{k'}} \frac{2\pi}{\pi-\arctan{k}} \zeta(z,v) dz) dv \right)\\ 
\nonumber
 -\frac{\arctan{k}}{2\pi}&\left(\int_{-\infty}^{0} (\int_{-\frac{v}{k'+k}}^{-\frac{v}{k'}} \frac{2\pi}{\arctan{k}}\zeta(z,v) dz) \ln(\int_{-\frac{v}{k'+k}}^{-\frac{v}{k'}} \frac{2\pi}{\arctan{k}}\zeta(z,v) dz) dv+ \right.\\
 \label{eq:giantentropy}
& \hspace*{0.3cm} \left.\int_{0}^{+\infty} (\int_{-\frac{v}{k'}}^{-\frac{v}{k'+k}} \frac{2\pi}{\arctan{k}}\zeta(z,v) dz) \ln(\int_{-\frac{v}{k'}}^{-\frac{v}{k'+k}} \frac{2\pi}{\arctan{k}}\zeta(z,v) dz) dv \right),
\end{align}
where $\zeta(z,v)$ is the function $\phi(v+k'z)\phi(z)$.

As for the class-conditional MI values that involve $Z$, ${\rm MI}(Z,X|C_k)={\rm MI}(Z,X-k'Y|C_k)={\rm MI}(Z,X_{{\rm disc}}|C_k)=0$. This requires checking that pairwise class-conditional independence holds for the three involved pairs. This follows, as argued in Section \ref{subsec:setting}, from the fact that $Z$ is independent of the pair composed by $C_k$ and any other input feature. 

\section{Calculations of MBR values in Section \ref{subsec:ordervariables}}
\label{app:B}

In this section, we start by obtaining the value of ${\rm MBR}(C_k,\{X\})$. We then prove that ${\rm MBR}(C_k,\{X,Z\})={\rm MBR}(C_k,\{X\})$.

Concerning the computation of ${\rm MBR}(C_k,\{X\})$, the $(C,\{X\})$ Bayes classifier assigns $x$, $x\in\mathcal X$, to $1$ if and only if, recall \eqref{eq:otpmequation},
\begin{equation*}
\frac{f_{X|X+kY<0}(x)}{f_{X|X+kY\geq 0}(x)}\leq 1.
\end{equation*}
The required densities $f_{X|X+kY\geq 0}$ and $f_{X|X+kY<0}$ are known from Appendix \ref{app:A2}. We take this into account to re-write the expression above as
\begin{equation*}
\frac{2\exp\left\{\frac{-x^2}{2\sigma^2_{X}\sigma^2_{X+kY}}\right\}\frac{1}{\sqrt{2\pi \sigma_{X}}}\Phi\left(-\frac{\rho \sqrt{\sigma^2_{X+kY}} x}{\sqrt{\sigma^2_{X}}}\frac{1}{\sqrt{\sigma^2_{X+kY}(1-\rho^2)}}\right)}{2\exp\left\{\frac{-x^2}{2\sigma^2_{X}\sigma^2_{X+kY}}\right\}\frac{1}{\sqrt{2\pi \sigma_{X}}}\Phi\left(\frac{\rho \sqrt{\sigma^2_{X+kY}} x}{\sqrt{\sigma^2_{X}}}\frac{1}{\sqrt{\sigma^2_{X+kY}(1-\rho^2)}}\right)}\leq 1,
\end{equation*}
where $\sigma_X^2=1$ is the variance of $X$, $\sigma_{X+kY}^2=1+k^2$ is the variance of $X+kY$, and $\rho=\frac{1}{k}$ is the correlation between $X$ and $X-k'Y$.

Many terms cancel out, and we get simply
\begin{equation*}
\Phi\left(-\frac{x}{k(1-\frac{1}{k^2})}\right)\leq\Phi\left(\frac{x}{k(1-\frac{1}{k^2})}\right).
\end{equation*}
As $\Phi$ is a non-decreasing function, this condition is the same as
\begin{equation*}
-\frac{x}{k(1-\frac{1}{k^2})}\leq \frac{x}{k(1-\frac{1}{k^2})}.
\end{equation*}
After further cancellations of the denominators, we simply obtain the condition $-x\leq x$, which holds if and only if $x\geq 0$. As a result, the classifier assigns $x$ to $1$ if $x$ is non-negative, and to $0$ otherwise; note that the classifier applied to $X$ gives $X_{{\rm disc}}$. 

Therefore, ${\rm MBR}(C_k,\{X\})$ depends on the angle, in the two-dimensional space defined by the pair $(X,Y)$, between the lines associated with $X$ and $X+kY$, $\arctan{k}$. Note that the only knowledge of $X$ needed concerns the value that $X_{{\rm disc}}$ takes; recall \eqref{eq:xdisc}. As a result, Figure \ref{fig:midisc} also illustrates the regions where a wrong classification will occur, which are the regions $(a)$ and $(b)$. The probabilities associated with the different regions have been given in \eqref{eq:jointdiscprobs}, allowing us to obtain
\begin{equation*}
{\rm MBR}(C_k,\{X\})=2 \frac{\arctan{k}}{2\pi}=\frac{\arctan{k}}{\pi}.
\end{equation*}

We now verify that ${\rm MBR}(C_k,\{X,Z\})={\rm MBR}(C_k,\{X\})$. By \eqref{eq:otpmequation}, the $(C_k,\{X,Z\})$ Bayes classifier associated with the MBR takes the value $1$ if and only if
\begin{equation*}
\frac{f_{(X,Z)|X+kY<0}(x,z)}{f_{(X,Z)|X+kY\geq 0}(x,z)}\leq 1.
\end{equation*}
Using the facts that $Z$ and $X$ are class-conditionally independent and that $Z$ is independent of $C_k$, the condition above reduces to
\begin{equation*}
\frac{f_{X|X+kY<0}(x)}{f_{X|X+kY\geq 0}(x)}\leq 1.
\end{equation*}
As a result, the points $(x,z)$ assigned by the classifier to the value $1$ are those that verify $x\geq 0$. As a result, ${\rm MBR}(C_k,\{X,Z\})={\rm MBR}(C_k,\{X\})$.